\documentclass[twoside]{article}

\usepackage[accepted]{aistats2025}

\usepackage{microtype}
\usepackage{graphicx}
\usepackage{subfigure}
\usepackage{subcaption}
\usepackage{booktabs} %
\newcommand{\txt}[1]{{\textsf{#1}}}
\newcommand{\tbm}[1]{\textbf{\textsf{#1}}} %

\newcommand\subsetsim{\mathrel{%
  \ooalign{\raise0.2ex\hbox{$\subset$}\cr\hidewidth\raise-0.8ex\hbox{\scalebox{0.9}{$\sim$}}\hidewidth\cr}}}
\usepackage{amsmath}
\usepackage{amssymb}
\usepackage{mathtools}
\usepackage{amsthm}

\usepackage{hyperref}
\usepackage[capitalize,noabbrev]{cleveref}
\theoremstyle{plain}
\newtheorem{theorem}{Theorem}[section]

\newtheorem{lemma}[theorem]{Lemma}

\newtheorem{definition}[theorem]{Definition}

\newtheorem*{objective*}{Objective}
\newtheorem{remark}[theorem]{Remark}

\usepackage[textsize=tiny]{todonotes}

\newcommand\Set[1]{\mathbb{#1}} %
\newcommand{\tsgn}[1]{{#1}}%
\renewcommand{\exp}[1]{\operatorname{e}^{#1}} %

\usepackage{comment}

\usepackage{bm}
\newcommand{\naught}{{_{{0\!}}}}

\let\normalint\int %
\def\int{\displaystyle\normalint} %
\usepackage{scalerel}

\newcommand{\overwrite}[3][blue]{%
  \settowidth{\overwritelength}{$#2$}%
  \ifdim\overwritelength<\minimumoverwritelength%
    \setlength{\overwritelength}{\minimumoverwritelength}\fi%
  \stackrel
    {%
      \begin{minipage}{\overwritelength}%
        \color{#1}\centering\small #3\\%
        \rule{1pt}{9pt}%
      \end{minipage}}
    {\colorbox{#1!20}{\color{black}$\displaystyle#2$}}}

\renewcommand{\text}[1]{\textnormal{#1}}

\newcommand{\eig}{\lambda}

\newcommand{\operator}[1]{\mathcal{#1}}
\newcommand{\raum}[1]{\mathcal{#1}}

\newcommand{\Complex}{\ensuremath{\mathbb{C}}}
\newcommand{\RKHS}{\ensuremath{\raum{H}}}

\newcommand{\Sobo}[2]{\raum{C}}%

\newcommand{\ArminKernel}[1][]{Armin-Kernel}
\newcommand{\ArminKernelshort}[1][]{A}

\renewcommand{\d}[1]{\ensuremath{\operatorname{d}\!{#1}}}

\newcommand{\KEqvar}[1]{{{#1}}_\eig}
\newcommand{\KEop}[2]{{{\mathcal{E}}^{{#1}}_{{#2}}}}

\usetikzlibrary{arrows,automata, calc, positioning}
\usepackage{float}
\usepackage{adjustbox}
\usepackage{wrapfig}
\usepackage{subcaption} 
\usepackage{cprotect}
\usepackage{etoc}
\usepackage{titlesec}
\usepackage{enumitem}

\usepackage{tabularx}
\usepackage{multirow}
\makeatother

\usepackage{nicefrac}
\usepackage{mathrsfs}  

\makeatletter
\newenvironment{taggedsubequations}[1]
 {%
  \addtocounter{equation}{-1}%
  \begin{subequations}%
  \def\@currentlabel{#1}%
  \renewcommand{\theequation}{{#1}\alph{equation}}%
 }
 {\end{subequations}}
\makeatother %

\newcommand{\RR}{\mathbb{R}}    %

\usepackage{pgfplots}

\usepackage{tikz}
\usepackage{tikz-3dplot}
\pgfmathdeclarefunction{normal}{2}{%
\pgfmathparse{1/(#2*sqrt(2*pi))*exp(-((x-#1)^2)/(2*#2^2))}%
}
\pgfplotsset{compat=1.10}
\usepackage{pgfplotstable}
\usepackage{datatool}

\usepackage[usestackEOL]{stackengine}
\usepackage{enumitem}%

\makeatletter
\def\namedlabel#1#2{\begingroup
    #2%
    \def\@currentlabel{#2}%
    \phantomsection\label{#1}\endgroup
}
\makeatother

\newcommand{\GP}{\mathcal{GP}}

\newcommand{\transpose}{^\mathrm{\textsf{\tiny T}}}
\newcount\colveccount
\newcommand*\colvec[1]{
        \global\colveccount#1
        \begin{bmatrix}
        \colvecnext
}
\def\colvecnext#1{
        #1
        \global\advance\colveccount-1
        \ifnum\colveccount>0
                \\
                \expandafter\colvecnext
        \else
                \end{bmatrix}
        \fi
}

\newcommand{\inv}{^{-1}}

\newcommand{\vk}{{\bm{k}}}
\newcommand{\vm}{{\bm{m}}}

\newcommand{\vu}{{\bm{u}}}

\newcommand{\vz}{{\tilde{\bm{z}}}}

\newcommand{\BigO}{\mathcal{O}}

\newcommand{\Kuu}{\bm{K}_{\bm{uu}}}
\newcommand{\Kud}{\vk_\vu\left(\cdot\right)}
\newcommand{\Kfu}{\bm{K}_{\bm{yu}}}
\newcommand{\Kdu}{\vk\transpose_\vu\left(\cdot\right)}
\newcommand{\kdd}{k(\cdot, \cdot)}
\newcommand{\Kff}{\bm{K}_{\bm{yy}}}

\usepackage{thmtools,thm-restate}

\usepackage{pifont}

\usepackage{amsmath,amsfonts,bm}

\def\eqref#1{(\ref{#1})}

\def\1{\bm{1}}

\def\mS{{\bm{S}}}

\def\mZ{{\tilde{\bm{Z}}}}

\DeclareMathAlphabet{\mathsfit}{\encodingdefault}{\sfdefault}{m}{sl}
\SetMathAlphabet{\mathsfit}{bold}{\encodingdefault}{\sfdefault}{bx}{n}

\DeclareMathOperator*{\argmin}{arg\,min}

\newcommand{\cmark}{\ding{51}}
\newcommand{\xmark}{\ding{55}}

\usepackage{dirtytalk} %
\usepackage{hyperref} %

\hypersetup{
    colorlinks,
    linkcolor={black},
    citecolor={black},
    urlcolor={black}
}
\usepackage{booktabs} %

\usepackage[round]{natbib}
\usepackage{xr-hyper}

\begin{document}

\runningauthor{P. Bevanda, M. Beier, A. Lederer, A. Capone, S. Sosnowski, S. Hirche}

\twocolumn[

\aistatstitle{Koopman-Equivariant Gaussian Processes}

\aistatsauthor{Petar Bevanda$^{*}$ \\ TU Munich \And Max Beier$^{*}$ \\ TU Munich
\And Alexandre Capone \\ CMU Robotics Institute \AND Stefan Sosnowski \\ TU Munich \And Sandra Hirche \\ TU Munich \And Armin Lederer \\ ETH Z\"{u}rich 
}
\aistatsaddress{} 
]

\begin{abstract}
Credible forecasting and representation learning of dynamical systems are of ever-increasing importance for reliable decision-making. To that end, we propose a family of Gaussian processes (GP) for dynamical systems with linear time-invariant responses, which are nonlinear only in initial conditions. This linearity allows us to tractably quantify forecasting and
representational uncertainty, simultaneously alleviating the challenge of computing the distribution of trajectories from a GP-based dynamical system and enabling a new probabilistic treatment of learning Koopman operator representations. Using a trajectory-based equivariance -- which we refer to as \textit{Koopman equivariance} -- we obtain a  GP model with enhanced generalization capabilities. To allow for large-scale regression, we equip our framework with variational inference based on suitable inducing points. Experiments demonstrate on-par and often better forecasting performance compared to kernel-based methods for learning dynamical systems.
\end{abstract}

\section{INTRODUCTION}
Learning predictive models for forecasting dynamic systems is a challenging task due to complex and often unknown interactions between quantities of interest \citep{Brunton2019Data-DrivenEngineering}. The great utility of such models helps advance various different fields such as fluid mechanics \citep{kundu2015fluid}, molecular biology \citep{protFold}, robotics \citep{Billard2022LearningApproach} or safety-constrained decision making \citep{Hewing/annurev-control-090419-075625, Brunke/annurev-control-042920-020211}.
Dynamical system descriptions commonly require simulation for forecasting and uncertainty propagation, which can be difficult for non-parametric data-driven models \citep{pmlr-v120-hewing20a,TBpredGP}. 
In most real-world applications involving dynamical systems,  
measurements often come in the form of sequential one-step transition data that is sampled arbitrarily and potentially non-uniformly. Furthermore, there is often a certain regularity in the evolution of quantities of interest \citep{pmlr-v202-bilos23a} across domains \citep{SEZER2020106181,DEB2017902,Lim2021}, making it important to impose structure that discourages temporal fluctuations. 
To account for these different challenges in modeling dynamical systems, the choice of \textit{representations} when learning from data becomes a deciding factor in the difficulty of forecasting as well as inference, especially when modeling complex phenomena \citep{Mezic2004ComparisonBehavior} or long time-series \citep{DBLP:conf/iclr/GuGR22}.
In this paper, we focus on non-parametric learning paradigms, emphasizing \textit{uncertainty quantification} and \textit{forecasting simplicity}. In particular, we study the interplay between Gaussian processes \citep{Rasmussen2006} and effective dynamical system linearizations based on Koopman operators \citep{KoopBook,Brunton2022ModernSystems}. A more exhaustive account of related work is delegated to the supplementary material.
\begin{table}[t!]%
    \caption{Nonlinear dynamics modeling from data}
    \label{tab:contribution}
    \footnotesize
    \centering
    \vspace{-2ex}
    \begin{tabular}{l|ccccc}
    \toprule
        Approach& LTI forecast & End-to-end & Bayesian\\
        \midrule
        GPs & \color{red!80!black}{\text{\xmark}}  & \color{green!80!black}{\text{\cmark}} & \color{green!80!black}{\text{\cmark}}\\
        Koopman & \color{green!80!black}{\text{\cmark}}  & \color{red!80!black}{\text{\xmark}} & \color{red!80!black}{\text{\xmark}} \\
        {{\tbm{this work}}} & \color{green!80!black}{\text{\cmark}} & \color{green!80!black}{\text{\cmark}} & \color{green!80!black}{\text{\cmark}} \\
        \bottomrule
    \end{tabular}
\end{table}

\textbf{Gaussian processes.~~}
Gaussian processes (GPs) \citep{Rasmussen2006} have the capability of inferring models with little structural prior knowledge: either by using so-called universal kernels \citep{Micchelli2006UniversalKernels} or placing a prior on a set of kernels and optimizing their likelihood \citep{Duvenaud2014}. 
In particular, their ability to quantify epistemic uncertainty has led to a common application in safety-critical control problems \citep{BerkenkampECC15,pmlr-v37-sui15,NIPS2017_766ebcd5,CuriCDC22,Baumann2021,Khosravi2023bo,As2024,Polymenakos2020,Lederer2021GaussianApplications}. Commonly employed as single-step predictors, GP models necessitate approximations for predicting probability distributions that go beyond a single time-step into the future. Thus, dealing with multi-step prediction often relies on iterative sampling-based approaches \citep{Bradford2019,pmlr-v120-hewing20a, TBpredGP} that are generally computationally expensive.
Alternatively, one can employ methods of reduced computational complexity, such as Taylor approximations \citep{Girard2003} or exact moment matching \citep{Deisenroth2011PILCO:Search}. However, such approaches deliver no accuracy guarantees for long-term forecasts. Notably, one can avoid approximate uncertainty propagation via multitask GPs models \citep{Bonilla2007} that use a collection of ``condensed models", one for each of the prediction steps \citep{Bradford2019,Pfefferkorn2022}, or employ a single contextual kernel defined over a joint spatio-temporal domain \citep{pmlr-v151-zenati22a, Li2024STkernel}.\looseness=-1

\textbf{Koopman operator-based learning.~~}
The linearity of Koopman operators and the forecasting simplicity of \textit{linear time-invariant} (LTI) models stemming from their eigendecompositions has led to their increasing popularity in learning dynamical systems \citep{Bevanda2021KoopmanControl,Otto2021AnnualSystems,Brunton2022ModernSystems}. 
Nevertheless, existing LTI predictors based on operator regression are limited to dissecting long-term components of ergodic dynamics \citep{Korda2018OnOperator,Klus2020EigendecompositionsSpaces,Kostic2022LearningSpaces,Kostic2023KoopmanEigenvalues}. While this approach is extremely powerful for stationary data and reversible dynamics, almost all real-world dynamical systems are irreversible and often even nonstationary \citep{Wu2020}.
Thus, an increasing amount of methods considers kernels that are \textit{dynamics-informed} \citep{Zhao_2016,BERRY2016439,Banisch2017,ALEXANDER2020132520,Burov2021,DUFEE2024134044}. By plugging samples of the dynamics from sequential data into the kernel itself, eigenfunctions of Koopman operators can be directly accessed for both ergodic \citep{DUFEE2024134044} and transient settings \citep{KKR_neurips2023}.~
While the latter has generalization and consistency guarantees, fully tractable representational uncertainty is impossible due to a two-stage regression approach \citep{HarmlessEcon,Wang2022}. Still, the existing Koopman operator-based learning approaches offer no epistemic uncertainty bounds, principled model selection or handling of observation noise.

In this work, we present \textbf{K}oopman-\textbf{E}quivariant \textbf{G}aussian \textbf{P}rocesses (KE-GPs), the first universal GP models with fully tractable and closed-form confidence bounds for multi-step prediction. By leveraging latent dynamics, our model provides simple LTI responses as a nonlinear function of the initial condition. Strikingly, our GP model provides enhanced generalization compared to existing methods due to intrinsic symmetries (Koopman-equivariants). Furthermore, it delivers continuous-time posteriors without requiring time-derivative data. KE-GPs allow for tractable \textit{simultaneous characterization of both forecasting and representational uncertainty} – alleviating a traditional
challenge of GPs and enabling a novel probabilistic treatment of learning dynamics representations.\footnote{{\textbf{Notation:} Denote the joint data distribution as $P(d z \times d x \times d y)$, its marginal distributions as $P(d x), P(d z)$, etc., and their support as $\mathcal{X}, \mathcal{Z}$. For functions of observed variables (e.g., $\bm{x}$ or z), $\|\cdot\|_2$ denotes the $L_2$ norm w.r.t. the respective marginal data distribution. $\|\cdot\|_{\infty}$ denotes the $L_{\infty}$ norm. We use the notation $[m]:=\{1, \ldots, m\}$. Boldface $(\bm{x}, \bm{y}, \bm{z})$ emphasizes the denotation of random variables. For any kernel $k, \mathcal{G} \mathcal{P}(0, k)$ refers to the ``standard Gaussian process" \citep{vanderVaart2008} with zero mean, and covariance defined by $k . \lesssim, \gtrsim, \asymp$ represent (in)equalities up to constants; the hidden constants will not depend on any sample size. $\tilde{\mathcal{O}}(\cdot)$ denotes inequality up to logarithm factors.}}

\textbf{Organization\footnote{Proofs of theoretical results are in the supplemental.}.}
In Section \ref{sec:ProbStat} we introduce the necessary preliminaries together with our problem statement.
    Section \ref{sec:KEframework} includes the derivation of Koopman-equivariant Gaussian process models, including representation theory and dynamical properties. We then analyze the sample-complexity of our approach through an information-theoretical lens, in Section \ref{section:analysisofsamplecomplexity}.
To handle large datasets, in Section \ref{section:svigps}, we present our Koopman-equivariant inducing variables for scalable GP-based modeling using variational inference.
In Section \ref{sec:NumExp} we demonstrate the utility of our KE-GP approach through a comparison to existing GP and Koopman approaches for learning dynamical systems
including predator-prey ODE, datasets from realistic robotic simulators as well as real-world weather data. 
Finally, in Section \ref{sec:Concl}, we conclude and mention the limitations of the approach.\looseness=-1
\section{PROBLEM SETTING}\label{sec:ProbStat}
Our work builds upon the extensive literature on GPs, their interplay with linear operators, and the concept of Koopman-equivariance, which extracts informative ``latent states" of dynamics, i.e., Koopman operator eigenfunctions, based on trajectory data. The following covers the necessary prerequisites for setting up the interplay between GPs, linear operators, and intrinsic dynamical system symmetries.
\subsection{Gaussian Process Regression}\label{sec:BackgroundGPs}
A Gaussian process is a generalization of the Gaussian distribution. It specifies a distribution, such that any finite collection of random variables follows a joint Gaussian distribution, which can be interpreted as a distribution over functions $f:\mathbb{R}^n\rightarrow\mathbb{R}$ commonly denoted by $f(\cdot)\sim\GP(m(\cdot),k(\cdot,\cdot))$ \citep{Rasmussen2006}. This distribution is defined using a prior mean function $m:\mathbb{R}^n\rightarrow\mathbb{R}$ and a covariance function $k:\mathbb{R}^n\times\mathbb{R}\rightarrow\mathbb{R}_{0,+}$. The mean function $m(\cdot)$ includes prior models and is often set to $0$ in the absence of such information, which we also assume in the following. The covariance function $k(\cdot,\cdot)$ encodes more abstract prior knowledge, such as symmetries and smoothness of the sample functions.

Given a dataset $\mathbb{D}_N=\{\bm{z}^{(i)},y^{(i)} \}_{i\in[N]}$ with training targets $y^{(i)}=f(\bm{z}^{(i)})+\omega^{(i)}$ perturbed by i.i.d. Gaussian noise $\omega^{(i)}\sim\mathcal{N}(0,\sigma_{\txt{on}}^2)$, we place a Gaussian process prior $\GP(0,k(\cdot,\cdot))$ on the unknown function $f(\cdot)$ to infer a model. This is straightforwardly achieved by computing the posterior distribution given the training data, which is Gaussian at each test point $\bm{z}\in\mathbb{R}^n$. We can then compactly express the posterior as $p(f(\bm{z})|\mathbb{D}_N)=\mathcal{N}(\mu(\bm{z}),\sigma^2(\bm{z}))$, where 
\begin{align}\label{eq:GPbase}
 \textstyle   \mu(\bm{z})&=\bm{k}^{\intercal}(\bm{z})(\bm{K}+\sigma_n^2\bm{I})^{-1}\bm{y},\\
    \sigma^2(\bm{z})&=k(\bm{z},\bm{z})-\bm{k}^{\intercal}(\bm{z})(\bm{K}+\sigma_{\txt{on}}^2\bm{I}_N)^{-1}\bm{k}(\bm{z}),
\end{align}
with $k_i(\bm{z})=k(\bm{z},\bm{z}^{(i)})$, $K_{ij}=k(\bm{z}^{(i)},\bm{z}^{(j)})$ and $\bm{y}^\intercal=[y^{(1)}\ \cdots\ y^{(N)}]$. In addition to inferring the posterior distribution, we use the training data to optimize the hyperparameters that arise from the kernel parameterization. This is enabled by the probabilistic approach to the regression problem, which allows us to choose the hyperparameters by minimizing the negative log-likelihood $\textstyle  -\log (p(\bm{y}|\bm{Z}))= \textstyle \nicefrac{1}{2}\bm{y}^{\intercal}(\bm{K}+\sigma_{\txt{on}}^2\bm{I}_N)^{-1}\bm{y}+\nicefrac{1}{2}\log(\det(\bm{K}+\sigma_{\txt{on}}^2\bm{I}_N))+\nicefrac{N}{2}\log(2\pi)$.

\subsection{System Class \& Modeling Approach}
\textbf{System class.~~} We consider state-space models
\begin{taggedsubequations}{\txt{SSM}}\label{eq:SSmodel}
\begin{align}
\text{(dynamics)}  \qquad  \dot{\bm{x}}&=\bm{f}
(\bm{x}), \quad \bm{x} \in \Set{X} \subset \RR^n, \quad \label{eq:SSMdyn}\\
\text{(output)}   \qquad    y&=h(\bm{x}) \in \RR, \quad \label{eq:SSMout}
\end{align}
\end{taggedsubequations}%
with a well-defined flow $\bm{F}_{t}(\bm{x}{\naught}):=\normalint_0^{t} \bm{f}(\bm{x}(\tau)) d\tau$ that requires local Lipschitz continuity of $\bm{f}$, which is natural to physical systems that often evolve ``smoothly''.
The canonical forecasting model for \eqref{eq:SSmodel} is $y(t,\bm{x}_0) := h_t(\bm{x}_0)\equiv h \circ \bm{F}_{t}(\bm{x}_0)$. In practice, a numerical integration scheme is usually required to solve the integral for a shorter time-interval ${\Delta}{t}$, such that the actual forecast becomes an $H$-fold composition of nonlinear maps $\textstyle y(t,\bm{x}_0) \approx \textstyle h \circ \bm{F}_{{\Delta}{t}} \circ \cdots \circ \bm{F}_{{\Delta}{t}}(\bm{x}_0)$ with $H = t/{\Delta}{t} \in \Set{N}$.

\textbf{Spectral dynamics modeling.~~}
To decompose nonlinear dynamics into simple linear factors and avoid approximate integration schemes, one can utilize the fact that the composition of a function $h$ with the flow $\bm{F}_t$ can be replaced by a linear, \textit{Koopman}, operator $\mathcal{A}_{t}: \RKHS^\prime\rightarrow \RKHS$ with $[\mathcal{A}_{t}h](\bm{x}_{0}):=  h_t(\bm{x}_{0}):=h(\bm{x}_{t})$~\citep{Koopman1931HamiltonianSpace,Cvitanovic2016Chaos:Quantum}. The usefulness of linear operators lies in their ability to forecast any $h \in \RKHS$ in terms of a spectral decomposition~\citep{Weidmann1980}\looseness=-1
\begin{align}\label{eq:KoopObs}
    [\mathcal{A}_{t}h](\bm{x}_0)  = \sum^{\infty}_{j=1}\underbrace{\exp{\eig_j t}\vphantom{g^\prime_j,h}}_{\text{\tiny dynamics}}\langle\underbrace{ g^\prime_j,h}_{\text{\tiny mode}}\rangle \underbrace{g_j(\bm{x}_0)}_{\text{\tiny (eigen)feature}},\tag{\txt{KMD}}
\end{align}
where the dynamics are parameterized by eigenvalues $\{\eig_j(\mathcal{A}_t)\}^{\infty}_{j=1} \in \Set{C}$ while $\RKHS^\prime:=\operatorname{span}(\{g^\prime_j\}^{\infty}_{j=1})$ span an auxiliary and $\RKHS := \operatorname{span}(\{g_j\}^{\infty}_{j=1})$ the main representation hypothesis. Under mild conditions \textit{Koopman mode decomposition} \eqref{eq:KoopObs} exists and is dense in $C(\Set{X})$~\citep{Korda2020OptimalControl}, cf. \citet{KKR_neurips2023} and references therein. Learning a finite \eqref{eq:KoopObs} from data, up to a re-scaling of modes, amounts to learning a $D$-dimensional representation $\RKHS_D$\footnote{$\RKHS^{\prime}_D=\RKHS_D$ holds w.l.o.g. \citep{Korda2020OptimalControl}.}.%
\subsection{Problem Statement}
Given no knowledge of the Koopman operator or \eqref{eq:SSmodel}, our goal is to learn a finite-dimensional model for \eqref{eq:KoopObs} from initial-state and timestep pairs to future output values
\begin{align}\label{eq:data}
  \mathbb{D}_N=\{(\bm{x}^{(i)}_{0},t^{(i)}),  y^{(i)} \}_{i\in[N]}.
\end{align}
Our model should satisfy the following properties:
\begin{description}[style=multiline]
    \item[\namedlabel{itm:Trac}{(\tbm{D})}]  \textbf{Trajectory distributions in closed-form:} It corresponds to a  Gaussian process framework that models \eqref{eq:KoopObs} based on \eqref{eq:data}.
    \item[\namedlabel{itm:Eff}{(\tbm{E})}]  \textbf{Data-efficient for dynamical systems:} Allows a sample complexity reduction through the equivariance of \eqref{eq:KoopObs} w.r.t. past state trajectories.
   \item[\namedlabel{itm:Scale}{(\tbm{S})}]  \textbf{Scales to large-scale data:} Admits variational inference techniques for \eqref{eq:KoopObs} based on suitable inducing points.
\end{description}

The closed-form trajectory distributions \ref{itm:Trac} of our proposed framework allow for $\bm{1)}$ continuous epistemic uncertainty over entire time-intervals (an important challenge in utilizing GP models for dynamical systems \citep{Ridderbusch2023}) and $\bm{2)}$ tractable Bayesian model selection -- both of which are absent in existing spectral dynamics modeling \citep{Brunton2022ModernSystems}. Additionally, successful inference on large datasets strongly depends on the availability of informative inducing points \citep{pmlr-v9-titsias10a}, which is particularly hard for high-dimensional inputs \citep{pmlr-v206-moss23a}. We propose to utilize the timeseries structure to induce an equivariant covariance using past trajectories that does not increase input dimensionality. Our approach can reduce the maximal information gain \ref{itm:Eff} \citep{Srinivas2012} and allows for effective variational inference for large-scale GP regression \ref{itm:Scale}.\looseness=-1

\section{KOOPMAN SPECTRAL GAUSSIAN PROCESSES}
\label{sec:KEframework}
Here we introduce a GP that respects the \eqref{eq:KoopObs} structure, building on the extensive literature on GPs \citep{Rasmussen2006,Duvenaud2014}, generalized additive models \citep{Krause2011ContextualOptimization,MojmírPHD} and their interplay with linear operators \citep{Matsumoto2024}.
\subsection{GP-Based Koopman Mode Decomposition}
Given a finite set of eigenvalues $\{\lambda_j\}_{j=1}^{|D|}$, the spectral decomposition \eqref{eq:KoopObs} induced by the Koopman operator can be straightforwardly translated into a structured GP model. For this, we assume independent GP priors $g_j(\cdot)\sim\mathcal{GP}(0,k_{g_j}(\cdot,\cdot))$ for the eigenfeatures $g_j(\cdot)$ and exploit the linearity of \eqref{eq:KoopObs} with the modes $\langle g^\prime_j,h\rangle$ equal to constant values, such that $y(\cdot,\cdot)$ follows a distribution $y(t,\bm{x}_0)\sim\mathcal{GP}(0,k_y((t,\bm{x}),(t',\bm{x}')))$ with 
\begin{align}
  \textstyle   k_y((t,\cdot),(t',\cdot'))&:= \sum\limits_{j \in [D]}{a}_j(t,t^\prime)k_{g_j}(\cdot,\cdot^\prime)\label{eq:SDK}\tag{${\txt{cov}_\text{\txt{SD}}}$},
\end{align}
where $ {a}_j(t,t^\prime) = \operatorname{e}^{\lambda_j t}\operatorname{e}^{\lambda^*_j t^\prime}$, and $k_{{g_j}}(\cdot,\cdot)$ can be arbitrary kernels. Conceptually, the kernel \eqref{eq:SDK} is akin to a simulation-induced kernel for linear systems \citep{CHEN2018109}, but now captures nonlinear dynamics \eqref{eq:SSmodel}. It exhibits the intuitive property that the spatial kernels $k_{{g_j}}(\cdot,\cdot)$ capture the representational uncertainty due to lifting of the dynamics to a higher dimensional space, in which the forecasting uncertainty evolves linearly according to the LTI features $\left\{{a}_j(t, t^{\prime})\right\}_{j \in [D]}$.
A temporal covariance ${a}_j(t, t^{\prime})$ with decay $|\lambda|$ close to zero will result in models with uniform uncertainty over time, whereas taking negative or positive decays will result in models with contracting or expanding variance over time, respectively. This allows for a straightforward encoding of prior knowledge about the temporal evolution of systems, e.g., stability. 

\textbf{Spectral hyperprior.~~}
While we generally do not have direct access to a sequence of eigenvalues $\{\lambda_j\}_{j=1}^{\infty}$, it is well known that this spectrum can be effectively covered by sampling a random distribution~\citep{KKR_neurips2023}. We can parameterize a spectral distribution, such that a high-likelihood representation for a finite series in \eqref{eq:SDK} can be obtained by integrating its parameters into Bayesian model selection.
To adopt such a spectral prior, we use the noise transfer (outsourcing) trick by \cite[Theorem 5.10]{kallenberg1997foundations} to model the eigenvalue distribution $p(\lambda) \approx \rho_{}(\bm{\vartheta})$.
This choice limits the number of required parameters since $\bm{\vartheta}$ has fewer parameters (degrees of freedom) than the number of eigenspaces $\|\bm{\vartheta}\|_0 \ll |D|$. Furthermore, it allows for the use of log-likelihood maximization just like with any other set of hyperparameters.
Note that the exact Koopman operator $\mathcal{A}_t$ can be approximated with arbitrary accuracy using a sufficiently large finite sequence $\{\lambda_j(\mathcal{A}_t)\}_{j=1}^D$ \citep{KKR_neurips2023}.

\subsection{Koopman-Equivariant Kernels}
While we can use arbitrary kernels for $k_{{g_j}}(\cdot,\cdot)$ in \eqref{eq:SDK}, such a dynamics-agnostic formulation does not exploit any properties of the Koopman operator underlying the spectral decomposition \eqref{eq:KoopObs} which gives rise to \eqref{eq:SDK}. In particular, Koopman operators $\mathcal{A}_t$ allow us to reverse the order of forward simulation and measurement function $h(\cdot)$ when determining the output $y$ at a time $t$, i.e., $h(\bm{F}_t(\bm{x}_0))=[\mathcal{A}_t h](\bm{x}_0)$. %
Considering only a single eigenvalue $\lambda_j$ of the spectral decomposition \eqref{eq:KoopObs} of the Koopman operator $\mathcal{A}_t$, this equivalence of representations induces a special class of functions, which we refer to as Koopman-equivariant. %
\begin{definition}[Koopman-equivariance]\label{def:KEIGS} 
    Let $[\tau_s,\tau_e] \subset \Set{R}$ be a compact subset of the time axis and $\mathcal{M}$ a manifold. 
    A map $\phi_{\eig}: \mathcal{M} \mapsto \Set{C}$ is called $ [\tau_s,\tau_e]_{\eig}$-{\em Koopman-equivariant} if 
    \begin{align}
        \phi_{\eig}\circ\bm{F}_t=\exp{\lambda t} \phi_\eig
    \end{align}
    on $\mathcal{M}$ for any $t \in [\tau_s,\tau_e]$.
\end{definition}
To ensure that the prior $\mathcal{GP}(0,k_{g_j})$ over eigenfeatures $g_j(\cdot)$ encodes Koopman-equivariance, observe that \cref{def:KEIGS} is a special case of the more general concept of subgroup equivariance \citep{pmlr-v139-satorras21a} adapted to Koopman operator open eigenfunctions \citep{Mezic2020SpectrumGeometry}. Since equivariance can be interpreted as a form of symmetry, this allows for the application of well-known techniques for the symmetrization of functions to ensure obtaining Koopman-equivariant functions. We follow the agnostic symmetrization approach of \citep{Kim2023,Nguyen2023}.
For this, we embed past state trajectories from time $\tau_s$ to $\tau_e$ into the input data, i.e.,
\begin{align}\label{eq:dataTraj}
  \mathbb{D}^{[\tau_s,\tau_e]}_N=\{(\bm{x}^{(i)}_{[\tau_s,\tau_e]},t^{(i)}),  y^{(i)} \}_{i\in[N]},
\end{align}
and exploit Definition~\ref{def:KEIGS} to obtain projections onto Koompan-equivariant subspaces of our hypothesis space. Notably, we can satisfy Koopman-equivariance in a simple and constructive manner, i.e., by taking an expectation, as we formalize in the following result.\looseness=-1

\begin{theorem}\label{thm:Symm}
Consider the symmetrization operator $\KEop{[\tau_s,\tau_e]}{\eig}: L_{\mu}^2(\mathcal{X}) \rightarrow L_{\mu}^2(\mathcal{X})$ defined as
\begin{align}\label{eq:EqvarEO}
    \KEop{[\tau_s,\tau_e]}{\eig} g := \mathbb{E}_{t \sim \mu{([\tau_s,\tau_e])}}\left[ \operatorname{e}^{-\lambda t} g\left(\bm{x}(t))\right)\right]
\end{align}
so that it is well-defined and self-adjoint. Then,
$\KEop{[\tau_s,\tau_e]}{\eig}$ maps $g$ to the unique solution of
\begin{align}
    \phi_{\eig} = \argmin_{\psi \in \mathcal{S}_\eig} \| g - \psi \|^2_{\mu}.
\end{align}
where $\mathcal{S}_\eig=\{g \in L_{\mu}^2(\mathcal{X}):\KEop{[\tau_s,\tau_e]}{\eig}  g=g\}$.
\end{theorem}

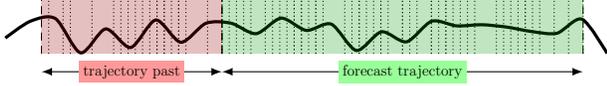
\begin{figure}[t!]
    \centering
    \usetikzlibrary{arrows.meta,
                quotes,             }
                \begin{tikzpicture}[scale=0.48, transform shape,
       every edge/.style = {draw, Straight Barb-Straight Barb},
every edge quotes/.style = {fill=white,font=\normalsize},
  doublearrow/.style={
    <->, 
    >=latex,
   every node/.style={fill=white!60!gray,font=\large}}]

\pgfmathsetseed{12}
\draw[very thick,draw = black,] plot[smooth,domain=0:10,samples=25] (\x/0.6,rnd);

\pgfmathsetmacro\FrameStart{1}
\pgfmathsetmacro\FrameW{5}
\pgfmathsetmacro\FrameOverlap{0.0}
\pgfmathsetmacro\FrameStep{\FrameW-\FrameOverlap}

\foreach[
  evaluate=\i as \y using {ifthenelse(mod(\i,2)==0,-0.5,-0.5)},
  evaluate=\i as \X using \i*\FrameStep+\FrameStart,
  count=\j
] \i in {0,...,1} {
  \draw [densely dashed] (\X,0) -- (\X,1.5);
}

\foreach[
  evaluate=\i as \y using {ifthenelse(mod(\i,2)==0,-0.5,-0.5)},
  evaluate=\i as \X using \i*\FrameStep+\FrameStart,
  count=\j
] \i in {0,...,0} {
  \draw [densely dashed] (\X+\FrameW,0) -- (\X+\FrameW,1.5);
  \draw [doublearrow] (\X,\y) -- node[fill=white!60!red]{trajectory past} (\X+\FrameW,\y);
}

\pgfmathsetmacro\FrameW{10}
\foreach[
  evaluate=\i as \X using \i*0.2+\FrameStart,
  count=\j
] \i in {1,...,70} {
\pgfmathsetmacro{\thenum}{int(random(0,1))}
  \draw [densely dotted] (\X+\thenum,0) -- (\X+\thenum,1.5);
}
\foreach[
  evaluate=\i as \y using {ifthenelse(mod(\i,2)==0,-0.5,-0.5)},
  evaluate=\i as \X using \i*\FrameStep+\FrameStart,
  count=\j
] \i in {1,...,1} {
  \draw [densely dashed] (\X+\FrameW,0) -- (\X+\FrameW,1.5);
  \draw [doublearrow] (\X,\y) -- node[fill=white!60!green]{forecast trajectory} (\X+\FrameW,\y);
}

\fill[fill=red!60!black,nearly transparent] (1,0) rectangle (6,1.5);

\fill[fill=green!60!black,nearly transparent] (6,0) rectangle (16,1.5);

\end{tikzpicture}
    \caption{Backward time equivariance interval (red) and the simulation-induced prediction horizon (green).}
    \label{fig:seq2seq}
\end{figure}

By construction, the symmetrization operator \eqref{eq:EqvarEO} renders every base function $g$ Koopman-equivariant on arbitrary past time intervals $[\tau_s,\tau_e]$, %
such that it allows us the straightforward design of Koopman-equivariant features $\phi_{\eig_j}(\cdot)$. To ensure causality of these features, we restrict ourselves to past trajectories, i.e., intervals $[\tau_s,0]$, which results in 
\begin{align}
    \phi_{\eig_j}(\bm{x}_{[\tau_s,0]}) = [\mathcal{E}_{\lambda_j}^{[\tau_s,0]} g](\bm{x}_0)).
\end{align}
Finally, since $\mathcal{E}_{\lambda_j}^{[\tau_s,0]}$ is a linear operator, we can exploit the closedness of Gaussian processes under linear operators \citep{Matsumoto2024} by placing a GP prior $g(\cdot)\sim\mathcal{GP}(0,k_{g}(\cdot,\cdot))$ on $g(\cdot)$ with arbitrary kernel $k_{g}(\cdot,\cdot)$, such that we obtain the Koopman-equivariant prior $\phi_{\eig_j}(\cdot)\sim\mathcal{GP}(0,k_{\phi_{\eig_j}}(\cdot,\cdot))$ with $k_{\phi_{\eig_j}}(\cdot,\cdot):=\mathcal{E}_{\eig_j}k_{g}(\cdot,\cdot^\prime)\mathcal{E}^{*}_{\eig_j}$, inducing a Koopman-equivariant spectral decomposition kernel 
\begin{align}\label{eq:KE-SDK}
  \textstyle   k^{\txt{KE}}_y((t,\cdot),(t',\cdot'))&:= \sum\limits_{j\in [D]}{a}_j(t,t^\prime)k_{\phi_{\eig_j}}(\cdot,\cdot),
\tag{${\txt{cov}_\text{\txt{KESD}}}$}
\end{align}
that exploits the full information in past trajectories in a structured way to allow predictions of the future evolution using \eqref{eq:KoopObs} as illustrated in Figure \ref{fig:seq2seq}.

\textbf{Practical considerations.~~}
In practice, our resolution of a trajectory is commonly limited by a sampling time, so we only have access to an empirical measure $\hat{\mu}$ for the expectation in \eqref{eq:EqvarEO}. Nevertheless, in most practical considerations and sufficiently regular trajectories, we will get a good sample-based approximation using quadrature so that $|\hat{\mathcal{E}}^{[\tau_s,0]}_{\eig}g-\KEop{[\tau_s,0]}{\eig}g|\approx 0$.\footnote{For a detailed treatment of this aspect cf. supplemental.}

\section{ANALYSIS OF SAMPLE COMPLEXITY}
\label{section:analysisofsamplecomplexity}
To analyze the sample complexity of regression \ref{itm:Eff}, we use the notion of \textit{information gain}, classical in the analysis of Gaussian processes  \citep{Srinivas2012Information-TheoreticSetting}. Our analysis allows us to put into perspective the sample complexity gains of using the proposed operator-theoretic GP w.r.t. more generic and less structured nonlinear models for dynamical systems. Thus, the following complexity study is a first in the literature.
Given the generalized additive structure of our \textit{spectral decomposition covariance} \eqref{eq:SDK}, we quantify the sample-complexity of learning using the well-established notion of maximal {information gain} 
\begin{align}\label{eq:IGbase}
\gamma^{\sigma}_N(k):=\sup _{\bm{x}_N \subseteq \Set{X}} I\left(\bm{y}_N ; y\right)=\frac{1}{2}\operatorname{log}|\bm{I}_N+{\sigma^{-2}}\bm{K}_N|,
\end{align}
that measures the interaction between the data, observation noise, and kernel. This quantity frequently appears in the analysis of the generalization or worst-case estimation error of Gaussian processes \citep{Krause2011ContextualOptimization,pmlr-v130-vakili21a}. The less complex the feature map of the kernel on the same $\Set{X}$, the smaller \eqref{eq:IGbase} will be, implying better statistical efficiency.

\subsection{Mercer Eigenvalues as a Proxy to Information Gain} To study the effect of general kernels on the complexity of learning, we will rely on Mercer's theorem \citep{Mercer1909FunctionsEquations} which states that for a well-behaved $k_x$, it can be expressed via the series expansion
\begin{align}\label{eq:Mercer}
   \textstyle k_x\left(\bm{x}, \bm{x}^{\prime}\right)=\sum_{j=1}^{\infty} \mu_j \varphi_j(\bm{x}) \varphi_j\left(\bm{x}^{\prime}\right),
\end{align}
such that $\{\sqrt{\mu_j}\varphi_j\}_{j=1}^{\infty}$ form an orthornormal basis of $L^2(\Set{X})$ with respect to a finite Borel measure\footnote{Generalization to more general input spaces is straightforward \citep{Steinwart2012MercersTO}.}. The complexity bounds we derive in this work will depend on \textit{how rapidly the eigenvalues $\{\mu_j\}_{j=1}^{\infty} \subseteq \Set{R}_{+}$ decay}. The decay of these eigenvalues is closely related to the complexity of the nonparametric model as well as the generalization properties of the posterior \citep{Micchelli1979DesignPF}. Generally, these eigenvalues decay faster for covariate distributions that are concentrated in a small volume and for kernels that give smooth mean predictors \citep{widom_asymptotic_1963,widom1964asymptotic}. Thus, the bounds we prove here verify the intuition that our Koopman-equivariant covariance can provide an improved finite-sample performance \ref{itm:Eff}. 
To analyze the effects of the induced Koopman equivariance on the sample complexity, some assumptions are needed:
\begin{description}[style=multiline, leftmargin=3em,font=\normalfont,parsep=0.01em]
    \item[\namedlabel{asm:Hreg}{(\txt{HR})}]  \textit{Regularity of the hypothesis:} \textbf{a)} $k_x$ is a Mercer kernel \citep{Mercer1909FunctionsEquations}. \textbf{b)} $\forall \bm{x}, \bm{x}^{\prime} \in$ $\Set{X},\left|k_x\left(\bm{x}, \bm{x}^{\prime}\right)\right| \leq \bar{k}$, for some $\bar{k}>0$ \textbf{c)} $\forall j \in \mathbb{N}, \forall \bm{x} \in \Set{X}$, $\left|\varphi_j(\bm{x})\right| \leq r$, for some $r>0$.
    \item[\namedlabel{asm:Sspec}{(\txt{WS})}]  \textit{The past trajectory interval $[\tau_s,\tau_e]$ and the set of initial conditions form a non-recurrent domain.} 
    \item[\namedlabel{asm:Areg}{(\txt{OR})}]  \textit{The operator $A_{t=1}:=A_1=\sum^{\infty}_{j=1}\exp{\eig_j }\KEop{[\tau_s,0]}{\eig_j}$ is a compact normal operator.} 
\end{description}
Generally, \ref{asm:Hreg} is a mild requirement and is fulfilled for continuous kernels on compact domains \citep{Wang2022}, while \ref{asm:Areg} is classical for limiting the ill-posedness of inverse problems \citep{Cavalier_2008}. \ref{asm:Sspec} is a mild technical assumption and allows for a well-specified symmetrization and Theorem \ref{thm:Symm} non-vacuous, as it ensures the existence of uncountably many functions satisfying Definition \ref{def:KEIGS} for any eigenvalue \citep[Appendix A]{KKR_neurips2023}.
For our  technical results, we differentiate between \textit{mildly} and \textit{severely} ill-posed setting, based on \textit{exponential} and \textit{polynomial} $\eig_j(A_1)$ decay rates, respectively.
\begin{remark}[Strict complexity reduction]\label{rmk:reduction}
    A direct consequence of well-specified equivariance \ref{asm:Sspec} is a guaranteed \emph{strict} reduction in the effective dimension \citep{Elesedy2021B}, which is known to equal the information gain up to logarithmic factors \citep{pmlr-v151-zenati22a}.\looseness=-1
\end{remark}
To study information gain rates for a general class of kernels (that includes our own), we will rely on recent results based on spectral decay properties of kernels \citep{pmlr-v130-vakili21a} and can state the following.
\begin{theorem}\label{th:infogain_asymp}
Consider the Mercer eigenvalues $\{\mu_j\}^{\infty}_{j=1}$ for $k_x$ and let Assumptions \ref{asm:Hreg},\ref{asm:Sspec} and \ref{asm:Areg} hold. Then $\exists \theta \geq 1$ for\vspace{-1em}
\begin{description}[style=multiline, leftmargin=3em,font=\normalfont,noitemsep]
    \item[\namedlabel{case:PD}{(\txt{Poly})}] $\eig_j(A_1)\lesssim j^{-p} \wedge \mu_j\lesssim j^{-a}$, $a>1$ or
    \item[\namedlabel{case:ED}{(\txt{Exp})}] $\eig_j(A_1)\lesssim \exp{-j^p} \wedge~\mu_j \lesssim  \exp{-j^b}$, $b>0$ so that
\end{description}
\[
    \gamma^{\sigma}_N\left(k^{\txt{KE}}_{y}\right)  \lesssim
    \tilde{\mathcal{O}}(\left(\gamma^{\sigma}_N(k_{x})\right)^{\nicefrac{1}{\theta}})
\]
where $\theta=\frac{\max\{2p,a\}}{a}$ \ref{case:PD} and $\theta=\frac{\max\{2p,b\}}{b}$ \ref{case:ED}.
\end{theorem}

The above result summarizes the rate gains from the Koopman-equivariant Gaussian process. In case the equivariance operator has a sufficiently strong singular value decay, i.e., $\theta > 1$, the information gain of our Koopman-equivariant GP with covariance \eqref{eq:KE-SDK} may be much smaller than for \eqref{eq:SDK}. Crucially, $\theta \geq 1$ is guaranteed, so a slow decay of the operator eigenvalues values will not deteriorate the already existing eigenvalue decay of $\{\mu_j\}^{\infty}_{j=1}$. As \ref{asm:Areg} plays the role of a feature extractor, our result suggests one could obtain a significantly improved rate when $\eig_j(A_1A_1^*)$ has a fast decay, signaling an induced RKHS with low complexity.\looseness=-1

The significance of the asymptotic rates for the maximum information gain when using \eqref{eq:KE-SDK} provided by Theorem~\ref{th:infogain_asymp} becomes clear when comparing the rates to the ones of other kernels as summarized in Table~\ref{tab:IGcomp}. For example, when using a na\"{i}ve contextual (spatio-temporal) kernel $k^{\txt{SE}}(t,t^\prime)\otimes k^{\txt{SE}}(\bm{x}_0,\bm{x}^\prime_i)$ \citep{pmlr-v151-zenati22a} defined over a joint spatio-temporal domain \citep{Li2024STkernel} via RBF kernels\footnote{The SE kernel is used for ease of exposition, but our results cover large classes of Mercer kernels.} $k^{\txt{SE}}$, it is well known that the maximum information gain behaves as $\tilde{\mathcal{O}}(\log(N)^{n+2})$. Due to the LTI features $a_j$ in \eqref{eq:SDK} for describing temporal correlations, the information gain for \ref{eq:SDK} reduces to $\tilde{\mathcal{O}}(\log(N)^{n+1})$ \citep{MojmírPHD}. In contrast, our proposed kernel  \ref{eq:KE-SDK} can exploit the inherent structure imposed by dynamical systems through Koopman-equivariance, such that a complexity of $\tilde{\mathcal{O}}(\log(N)^{\frac{n}{\theta}+1})$ is guaranteed when using SE kernels as the basis for $k_{\phi_{\lambda_j}}$ in \eqref{eq:KE-SDK}. Hence, for $\theta>1$, we virtually counteract the curse of spatial dimensionality that comes from the generic and measure-agnostic bounds on the eigenvalue decay for popular kernels \citep{pmlr-v75-belkin18a}. Note that, due to employing Koompan-equivariance (Theorem \ref{thm:Symm}), the sample complexity is not impeded by the length or time-discretization of a continuous-time trajectory, which sets \eqref{eq:KE-SDK} apart from kernels agnostic of the dynamical systems properties as illustrated in Table~\ref{tab:IGcomp}.\looseness=-1

\begin{table}[t!]%
    \caption{Worst-case information gain (w/o $\log$ factors) for universal RBF base kernel $k^{\txt{SE}}$. Under mild conditions, $\theta  \geq 1$ guarantees reduced sample complexity. The number of discretization steps necessary to handle trajectory inputs in na\"{i}ve kernels and \eqref{eq:SDK} is denoted by $|G|$.}
    \label{tab:IGcomp}
    \setlength{\tabcolsep}{5pt}
    \centering
    \scriptsize
    \vspace{-2ex}
    \begin{tabular}{l|ccc}
    \toprule
        ${\gamma^{\sigma}_N(\cdot)}$ & na\"{i}ve & \eqref{eq:SDK}  & \eqref{eq:KE-SDK} \\
        \midrule
       $\bm{x}_0$ &  $\tilde{\mathcal{O}}(\log(N)^{n{+}2})$ & $\tilde{\mathcal{O}}(\log(N)^{n+1})$ & --- \\
       $\bm{x}_{{\tau_s{,}0}}$ &$\tilde{\mathcal{O}}(\log(N)^{|G|n{+}2})$ & $\tilde{\mathcal{O}}(\log(N)^{|G|n{+}1})$ & $\tilde{\mathcal{O}}(\log(N)^{\frac{n}{\theta}{+}1})$\\
        \bottomrule
    \end{tabular}
\end{table}

\section{VARIATIONAL INFERENCE FOR KOOPMAN-EQUIVARIANT GPs}

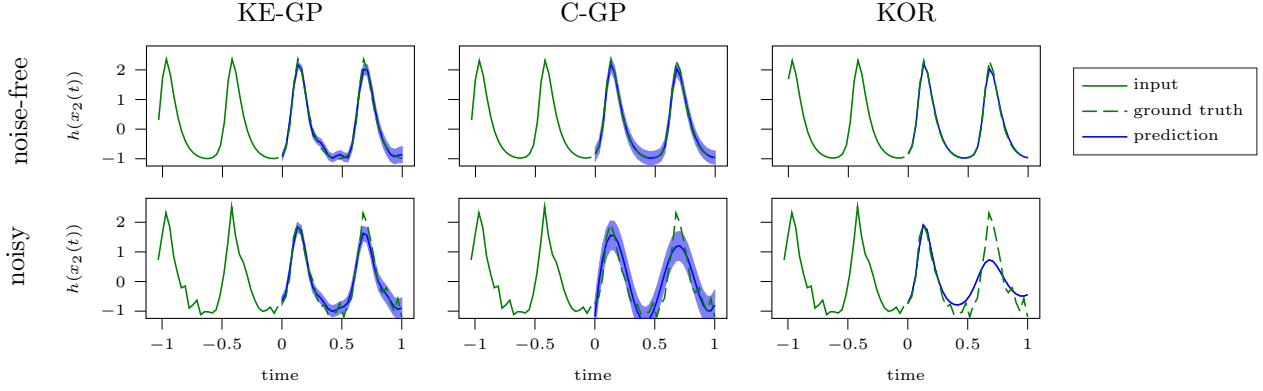
\begin{figure*}[t!]
    \setlength{\tabcolsep}{1pt}
    \renewcommand{\arraystretch}{0.25}
    \centering
    \begin{tabular}{cccccc} 
    \begin{tikzpicture}
            \node[rotate=90] at (0,0) {noise-free};
            \node at (0,-1.) {~};
    \end{tikzpicture}
    & ~&\begin{tikzpicture}
\definecolor{darkgray176}{RGB}{176,176,176}
\definecolor{green}{RGB}{0,128,0}
\definecolor{lightgray204}{RGB}{204,204,204}

\begin{axis}[
    width=.3\textwidth,
    height=.3\textwidth/1.618,
legend cell align={left},
legend style={
  fill opacity=0.8,
  draw opacity=1,
  text opacity=1,
  at={(0.91,0.5)},
  anchor=east,
  draw=lightgray204
},
tick align=outside,
tick pos=left,
title={},
x grid style={darkgray176},
xmin=-1.13387096774194, xmax=1.10161290322581,
yticklabel style = {font = \tiny},
xticklabel style = {font = \tiny},
xtick style={color=black},
y grid style={darkgray176},
ylabel={\tiny $h(x_2(t))$},
ymin=-1.24434151103344, ymax=2.80816290983449,
ytick style={color=black},
title = {KE-GP},
xticklabels={,,}
]
\path [draw=blue, fill=blue, opacity=0.5]
(axis cs:0,-1.02140400442368)
--(axis cs:0,-1.02140400442368)
--(axis cs:0.0322580645161294,-0.652495510938043)
--(axis cs:0.064516129032258,0.262354230312171)
--(axis cs:0.096774193548387,1.38587788384001)
--(axis cs:0.129032258064516,2.07107751409938)
--(axis cs:0.161290322580645,1.91304479941786)
--(axis cs:0.193548387096774,1.096045791632)
--(axis cs:0.225806451612903,0.200357004994826)
--(axis cs:0.258064516129032,-0.326932623751153)
--(axis cs:0.290322580645161,-0.495624209343708)
--(axis cs:0.32258064516129,-0.596688569965196)
--(axis cs:0.354838709677419,-0.808100150089652)
--(axis cs:0.387096774193548,-1.03672170704366)
--(axis cs:0.419354838709678,-1.10878441422885)
--(axis cs:0.451612903225806,-1.03331327373072)
--(axis cs:0.483870967741935,-0.99727272098956)
--(axis cs:0.516129032258065,-1.08882645506894)
--(axis cs:0.548387096774194,-1.10726816024921)
--(axis cs:0.580645161290323,-0.723733555461454)
--(axis cs:0.612903225806452,0.152487838763009)
--(axis cs:0.645161290322581,1.18784215908192)
--(axis cs:0.67741935483871,1.84425611957173)
--(axis cs:0.709677419354839,1.80857170010343)
--(axis cs:0.741935483870968,1.21010470785975)
--(axis cs:0.774193548387097,0.445552659859597)
--(axis cs:0.806451612903226,-0.170389730294337)
--(axis cs:0.838709677419355,-0.582415954342097)
--(axis cs:0.870967741935484,-0.87900913355136)
--(axis cs:0.903225806451613,-1.09089609316336)
--(axis cs:0.935483870967742,-1.16971564012116)
--(axis cs:0.967741935483871,-1.13517934879613)
--(axis cs:1,-1.13540273943964)
--(axis cs:1,-1.13540273943964)
--(axis cs:1,-0.589336779426036)
--(axis cs:0.967741935483871,-0.628004884616578)
--(axis cs:0.935483870967742,-0.685345208139557)
--(axis cs:0.903225806451613,-0.626077562117123)
--(axis cs:0.870967741935484,-0.433322283840301)
--(axis cs:0.838709677419355,-0.154140726890767)
--(axis cs:0.806451612903226,0.241992316668614)
--(axis cs:0.774193548387097,0.840905006264103)
--(axis cs:0.741935483870968,1.58841412615865)
--(axis cs:0.709677419354839,2.17141849752468)
--(axis cs:0.67741935483871,2.19181159353498)
--(axis cs:0.645161290322581,1.52129728516532)
--(axis cs:0.612903225806452,0.472803433036172)
--(axis cs:0.580645161290323,-0.418231775360382)
--(axis cs:0.548387096774194,-0.814981052272685)
--(axis cs:0.516129032258065,-0.805812210943319)
--(axis cs:0.483870967741935,-0.723072183173576)
--(axis cs:0.451612903225806,-0.766537276945688)
--(axis cs:0.419354838709678,-0.846024662513243)
--(axis cs:0.387096774193548,-0.775135242656734)
--(axis cs:0.354838709677419,-0.545246972883508)
--(axis cs:0.32258064516129,-0.337834021537275)
--(axis cs:0.290322580645161,-0.249386944092397)
--(axis cs:0.258064516129032,-0.089932232208076)
--(axis cs:0.225806451612903,0.433451550518775)
--(axis cs:0.193548387096774,1.32119318186608)
--(axis cs:0.161290322580645,2.12897982235537)
--(axis cs:0.129032258064516,2.28022951849028)
--(axis cs:0.096774193548387,1.58874655999785)
--(axis cs:0.064516129032258,0.460734887662643)
--(axis cs:0.0322580645161294,-0.458020404542072)
--(axis cs:0,-0.830871395563078)
--cycle;

\addplot [semithick, blue]
table {%
0 -0.926137699993381
0.0322580645161294 -0.555257957740058
0.064516129032258 0.361544558987407
0.096774193548387 1.48731222191893
0.129032258064516 2.17565351629483
0.161290322580645 2.02101231088661
0.193548387096774 1.20861948674904
0.225806451612903 0.316904277756801
0.258064516129032 -0.208432427979615
0.290322580645161 -0.372505576718053
0.32258064516129 -0.467261295751236
0.354838709677419 -0.67667356148658
0.387096774193548 -0.905928474850199
0.419354838709678 -0.977404538371049
0.451612903225806 -0.899925275338205
0.483870967741935 -0.860172452081568
0.516129032258065 -0.947319333006131
0.548387096774194 -0.961124606260949
0.580645161290323 -0.570982665410918
0.612903225806452 0.312645635899591
0.645161290322581 1.35456972212362
0.67741935483871 2.01803385655336
0.709677419354839 1.98999509881406
0.741935483870968 1.3992594170092
0.774193548387097 0.64322883306185
0.806451612903226 0.0358012931871385
0.838709677419355 -0.368278340616432
0.870967741935484 -0.65616570869583
0.903225806451613 -0.858486827640242
0.935483870967742 -0.92753042413036
0.967741935483871 -0.881592116706355
1 -0.86236975943284
};
\addlegendentry{Prediction}
\addplot [semithick, green]
table {%
-1.03225806451613 0.310677215876974
-1 1.72509317079231
-0.967741935483871 2.35958067123089
-0.935483870967742 1.85373346773129
-0.903225806451613 1.06436987498561
-0.870967741935484 0.39375284056191
-0.838709677419355 -0.0953849947985607
-0.806451612903226 -0.432355212837299
-0.774193548387097 -0.658219795153289
-0.741935483870968 -0.806783922492113
-0.709677419354839 -0.902169203461053
-0.67741935483871 -0.960134910896678
-0.645161290322581 -0.989464930765762
-0.612903225806452 -0.991868128129937
-0.580645161290323 -0.957961550613222
-0.548387096774193 -0.851288713427186
-0.516129032258065 -0.554026202604845
-0.483870967741935 0.243038106393077
-0.451612903225806 1.65036100981229
-0.419354838709677 2.36177948342126
-0.387096774193548 1.89633107767594
-0.354838709677419 1.10719751036095
-0.32258064516129 0.4266122896035
-0.290322580645161 -0.0722737325925517
-0.258064516129032 -0.416693101024244
-0.225806451612903 -0.64782458667783
-0.193548387096774 -0.80001455124028
-0.161290322580645 -0.897905560333097
-0.129032258064516 -0.957683656959815
-0.096774193548387 -0.988507495349605
-0.064516129032258 -0.992535488744642
-0.032258064516129 -0.961254120271588
};
\addlegendentry{Condition}
\addplot [semithick, green, dash pattern=on 5.55pt off 2.4pt]
table {%
0 -0.860613365987377
0.0322580645161294 -0.580085256941532
0.064516129032258 0.177945140645911
0.096774193548387 1.5730115929598
0.129032258064516 2.35986002626562
0.161290322580645 1.93810188327795
0.193548387096774 1.15045548523349
0.225806451612903 0.460068281182399
0.258064516129032 -0.0486707058542165
0.290322580645161 -0.400673957842
0.32258064516129 -0.637181665597286
0.354838709677419 -0.793075190688397
0.387096774193548 -0.893522844289506
0.419354838709678 -0.955142308935798
0.451612903225806 -0.987467578406421
0.483870967741935 -0.993096551558174
0.516129032258065 -0.964343694128448
0.548387096774194 -0.869419253002998
0.580645161290323 -0.604680161786228
0.612903225806452 0.115439309275864
0.645161290322581 1.49344249770474
0.67741935483871 2.3536498063136
0.709677419354839 1.97891203269848
0.741935483870968 1.19411427928786
0.774193548387097 0.494122687936025
0.806451612903226 -0.024568177750323
0.838709677419355 -0.384290661890491
0.870967741935484 -0.626285716196433
0.903225806451613 -0.785962214418609
0.935483870967742 -0.889018792725728
0.967741935483871 -0.952509769733059
1 -0.986345405974012
};
\addlegendentry{Target}
\legend{}
\end{axis}

\end{tikzpicture}
    & \begin{tikzpicture}
\definecolor{darkgray176}{RGB}{176,176,176}
\definecolor{green}{RGB}{0,128,0}
\definecolor{lightgray204}{RGB}{204,204,204}

\begin{axis}[
    width=.3\textwidth,
    height=.3\textwidth/1.618,
legend cell align={left},
legend style={
  fill opacity=0.8,
  draw opacity=1,
  text opacity=1,
  at={(0.91,0.5)},
  anchor=east,
  draw=lightgray204
},
tick align=outside,
tick pos=left,
title={},
x grid style={darkgray176},
xmin=-1.13387096774194, xmax=1.10161290322581,
yticklabel style = {font = \tiny},
xticklabel style = {font = \tiny},
xtick style={color=black},
y grid style={darkgray176},
yticklabels={,,},
ymin=-1.24434151103344, ymax=2.80816290983449,
ytick style={color=black},
title = {C-GP},
xticklabels={,,}
]
\path [draw=blue, fill=blue, opacity=0.5]
(axis cs:0,-1.05108515937071)
--(axis cs:0,-1.05108515937071)
--(axis cs:0.0322580645161294,-0.793160590795055)
--(axis cs:0.064516129032258,-0.0603062917985711)
--(axis cs:0.096774193548387,1.22882234949531)
--(axis cs:0.129032258064516,1.98426786044561)
--(axis cs:0.161290322580645,1.66841311569064)
--(axis cs:0.193548387096774,0.958310780161869)
--(axis cs:0.225806451612903,0.306394904418292)
--(axis cs:0.258064516129032,-0.212888207688193)
--(axis cs:0.290322580645161,-0.582703101588369)
--(axis cs:0.32258064516129,-0.835874367322634)
--(axis cs:0.354838709677419,-1.00203768354496)
--(axis cs:0.387096774193548,-1.10714748485308)
--(axis cs:0.419354838709678,-1.17140308919753)
--(axis cs:0.451612903225806,-1.19965554864374)
--(axis cs:0.483870967741935,-1.19530653556042)
--(axis cs:0.516129032258065,-1.15936290453452)
--(axis cs:0.548387096774194,-1.0803418364458)
--(axis cs:0.580645161290323,-0.800198044363484)
--(axis cs:0.612903225806452,0.00716944377067719)
--(axis cs:0.645161290322581,1.16755075412088)
--(axis cs:0.67741935483871,1.8172387957578)
--(axis cs:0.709677419354839,1.58508429550798)
--(axis cs:0.741935483870968,0.974585357927645)
--(axis cs:0.774193548387097,0.357310713603246)
--(axis cs:0.806451612903226,-0.15783710235564)
--(axis cs:0.838709677419355,-0.531319888487658)
--(axis cs:0.870967741935484,-0.798065930451727)
--(axis cs:0.903225806451613,-0.97783090509574)
--(axis cs:0.935483870967742,-1.09261986576696)
--(axis cs:0.967741935483871,-1.15862392448967)
--(axis cs:1,-1.18635885172024)
--(axis cs:1,-1.18635885172024)
--(axis cs:1,-0.741559354342373)
--(axis cs:0.967741935483871,-0.713829125691939)
--(axis cs:0.935483870967742,-0.647831583308663)
--(axis cs:0.903225806451613,-0.533053684862027)
--(axis cs:0.870967741935484,-0.353302013267092)
--(axis cs:0.838709677419355,-0.0865646795851854)
--(axis cs:0.806451612903226,0.286913791961969)
--(axis cs:0.774193548387097,0.802061551717969)
--(axis cs:0.741935483870968,1.41933709365362)
--(axis cs:0.709677419354839,2.0298323542427)
--(axis cs:0.67741935483871,2.26198510247366)
--(axis cs:0.645161290322581,1.61229702312937)
--(axis cs:0.612903225806452,0.45191606222082)
--(axis cs:0.580645161290323,-0.355451984459446)
--(axis cs:0.548387096774194,-0.635596974629444)
--(axis cs:0.516129032258065,-0.714619200204716)
--(axis cs:0.483870967741935,-0.750562600662023)
--(axis cs:0.451612903225806,-0.754910083090394)
--(axis cs:0.419354838709678,-0.726655065900912)
--(axis cs:0.387096774193548,-0.662396999060331)
--(axis cs:0.354838709677419,-0.557290313339431)
--(axis cs:0.32258064516129,-0.391130080617812)
--(axis cs:0.290322580645161,-0.137956495689864)
--(axis cs:0.258064516129032,0.231857750074283)
--(axis cs:0.225806451612903,0.751141053470489)
--(axis cs:0.193548387096774,1.40305874079444)
--(axis cs:0.161290322580645,2.11316494792416)
--(axis cs:0.129032258064516,2.4290291703633)
--(axis cs:0.096774193548387,1.6735944055097)
--(axis cs:0.064516129032258,0.384476746922106)
--(axis cs:0.0322580645161294,-0.348367611066797)
--(axis cs:0,-0.606285059749114)
--cycle;

\addplot [semithick, blue]
table {%
0 -0.828685109559914
0.0322580645161294 -0.570764100930926
0.064516129032258 0.162085227561768
0.096774193548387 1.4512083775025
0.129032258064516 2.20664851540445
0.161290322580645 1.8907890318074
0.193548387096774 1.18068476047816
0.225806451612903 0.528767978944391
0.258064516129032 0.00948477119304473
0.290322580645161 -0.360329798639116
0.32258064516129 -0.613502223970223
0.354838709677419 -0.779663998442196
0.387096774193548 -0.884772241956707
0.419354838709678 -0.94902907754922
0.451612903225806 -0.977282815867066
0.483870967741935 -0.972934568111222
0.516129032258065 -0.93699105236962
0.548387096774194 -0.857969405537623
0.580645161290323 -0.577825014411465
0.612903225806452 0.229542752995749
0.645161290322581 1.38992388862513
0.67741935483871 2.03961194911573
0.709677419354839 1.80745832487534
0.741935483870968 1.19696122579063
0.774193548387097 0.579686132660608
0.806451612903226 0.0645383448031647
0.838709677419355 -0.308942284036422
0.870967741935484 -0.57568397185941
0.903225806451613 -0.755442294978883
0.935483870967742 -0.87022572453781
0.967741935483871 -0.936226525090803
1 -0.963959103031306
};
\addlegendentry{Prediction}
\addplot [semithick, green]
table {%
-1.03225806451613 0.303603707030286
-1 1.69110343786114
-0.967741935483871 2.31351669120059
-0.935483870967742 1.81729572009365
-0.903225806451613 1.04295365427054
-0.870967741935484 0.385098409128619
-0.838709677419355 -0.0947311717972172
-0.806451612903226 -0.425288873159369
-0.774193548387097 -0.646855270285042
-0.741935483870968 -0.792592233996031
-0.709677419354839 -0.886162340589994
-0.67741935483871 -0.94302496518416
-0.645161290322581 -0.971796837088032
-0.612903225806452 -0.974154301795878
-0.580645161290323 -0.940892963775511
-0.548387096774193 -0.836250102269154
-0.516129032258065 -0.54464447358633
-0.483870967741935 0.237251764466004
-0.451612903225806 1.61779342401137
-0.419354838709677 2.31567366016837
-0.387096774193548 1.85908270088021
-0.354838709677419 1.08496628312198
-0.32258064516129 0.417332545446167
-0.290322580645161 -0.0720597150884879
-0.258064516129032 -0.409924810099038
-0.225806451612903 -0.636657881807129
-0.193548387096774 -0.785951683345303
-0.161290322580645 -0.881979834254542
-0.129032258064516 -0.940620358416416
-0.096774193548387 -0.970857621592344
-0.064516129032258 -0.974808962590712
-0.032258064516129 -0.944122876092924
};
\addlegendentry{Condition}
\addplot [semithick, green, dash pattern=on 5.55pt off 2.4pt]
table {%
0 -0.845397307426445
0.0322580645161294 -0.570207626173946
0.064516129032258 0.173397512723927
0.096774193548387 1.54191596046199
0.129032258064516 2.31379073013104
0.161290322580645 1.90005861134556
0.193548387096774 1.12740106214391
0.225806451612903 0.450151872143582
0.258064516129032 -0.0489058520877034
0.290322580645161 -0.394210509948319
0.32258064516129 -0.626217494675471
0.354838709677419 -0.779144378277501
0.387096774193548 -0.877680520951649
0.419354838709678 -0.938127372044891
0.451612903225806 -0.969837494186583
0.483870967741935 -0.975359348423147
0.516129032258065 -0.947153655602924
0.548387096774194 -0.854035619106571
0.580645161290323 -0.594334491918945
0.612903225806452 0.112081162395787
0.645161290322581 1.463861058815
0.67741935483871 2.30769869017498
0.709677419354839 1.94009214684081
0.741935483870968 1.17022903279026
0.774193548387097 0.483558226222138
0.806451612903226 -0.0252619931929575
0.838709677419355 -0.378138986826819
0.870967741935484 -0.615528894332488
0.903225806451613 -0.772166761385214
0.935483870967742 -0.873262181133951
0.967741935483871 -0.935544929855583
1 -0.968736676607071
};
\addlegendentry{Target}
\legend{}
\end{axis}

\end{tikzpicture}
    &\begin{tikzpicture}
\definecolor{darkgray176}{RGB}{176,176,176}
\definecolor{green}{RGB}{0,128,0}
\definecolor{lightgray204}{RGB}{204,204,204}

\begin{axis}[
    width=.3\textwidth,
    height=.3\textwidth/1.618,
legend cell align={left},
legend style={
  fill opacity=0.8,
  draw opacity=1,
  text opacity=1,
  at={(0.91,0.5)},
  anchor=east,
  draw=lightgray204
},
tick align=outside,
tick pos=left,
title={},
x grid style={darkgray176},
xmin=-1.13387096774194, xmax=1.10161290322581,
yticklabel style = {font = \tiny},
xticklabel style = {font = \tiny},
xtick style={color=black},
y grid style={darkgray176},
ymin=-1.24434151103344, ymax=2.80816290983449,
ytick style={color=black},
yticklabels={,,},
title = {KOR},
xticklabels={,,}
]
\addplot [semithick, blue]
table {%
0 -0.828685109559914
0.0322580645161294 -0.570764100930926
0.064516129032258 0.162085227561768
0.096774193548387 1.4512083775025
0.129032258064516 2.20664851540445
0.161290322580645 1.8907890318074
0.193548387096774 1.18068476047816
0.225806451612903 0.528767978944391
0.258064516129032 0.00948477119304473
0.290322580645161 -0.360329798639116
0.32258064516129 -0.613502223970223
0.354838709677419 -0.779663998442196
0.387096774193548 -0.884772241956707
0.419354838709678 -0.94902907754922
0.451612903225806 -0.977282815867066
0.483870967741935 -0.972934568111222
0.516129032258065 -0.93699105236962
0.548387096774194 -0.857969405537623
0.580645161290323 -0.577825014411465
0.612903225806452 0.229542752995749
0.645161290322581 1.38992388862513
0.67741935483871 2.03961194911573
0.709677419354839 1.80745832487534
0.741935483870968 1.19696122579063
0.774193548387097 0.579686132660608
0.806451612903226 0.0645383448031647
0.838709677419355 -0.308942284036422
0.870967741935484 -0.57568397185941
0.903225806451613 -0.755442294978883
0.935483870967742 -0.87022572453781
0.967741935483871 -0.936226525090803
1 -0.963959103031306
};
\addlegendentry{Prediction}
\addplot [semithick, green]
table {%
-1 1.69110343786114
-0.967741935483871 2.31351669120059
-0.935483870967742 1.81729572009365
-0.903225806451613 1.04295365427054
-0.870967741935484 0.385098409128619
-0.838709677419355 -0.0947311717972172
-0.806451612903226 -0.425288873159369
-0.774193548387097 -0.646855270285042
-0.741935483870968 -0.792592233996031
-0.709677419354839 -0.886162340589994
-0.67741935483871 -0.94302496518416
-0.645161290322581 -0.971796837088032
-0.612903225806452 -0.974154301795878
-0.580645161290323 -0.940892963775511
-0.548387096774193 -0.836250102269154
-0.516129032258065 -0.54464447358633
-0.483870967741935 0.237251764466004
-0.451612903225806 1.61779342401137
-0.419354838709677 2.31567366016837
-0.387096774193548 1.85908270088021
-0.354838709677419 1.08496628312198
-0.32258064516129 0.417332545446167
-0.290322580645161 -0.0720597150884879
-0.258064516129032 -0.409924810099038
-0.225806451612903 -0.636657881807129
-0.193548387096774 -0.785951683345303
-0.161290322580645 -0.881979834254542
-0.129032258064516 -0.940620358416416
-0.096774193548387 -0.970857621592344
-0.064516129032258 -0.974808962590712
-0.032258064516129 -0.944122876092924
};
\addlegendentry{Condition}
\addplot [semithick, green, dash pattern=on 5.55pt off 2.4pt]
table {%
0 -0.845397307426445
0.0322580645161294 -0.570207626173946
0.064516129032258 0.173397512723927
0.096774193548387 1.54191596046199
0.129032258064516 2.31379073013104
0.161290322580645 1.90005861134556
0.193548387096774 1.12740106214391
0.225806451612903 0.450151872143582
0.258064516129032 -0.0489058520877034
0.290322580645161 -0.394210509948319
0.32258064516129 -0.626217494675471
0.354838709677419 -0.779144378277501
0.387096774193548 -0.877680520951649
0.419354838709678 -0.938127372044891
0.451612903225806 -0.969837494186583
0.483870967741935 -0.975359348423147
0.516129032258065 -0.947153655602924
0.548387096774194 -0.854035619106571
0.580645161290323 -0.594334491918945
0.612903225806452 0.112081162395787
0.645161290322581 1.463861058815
0.67741935483871 2.30769869017498
0.709677419354839 1.94009214684081
0.741935483870968 1.17022903279026
0.774193548387097 0.483558226222138
0.806451612903226 -0.0252619931929575
0.838709677419355 -0.378138986826819
0.870967741935484 -0.615528894332488
0.903225806451613 -0.772166761385214
0.935483870967742 -0.873262181133951
0.967741935483871 -0.935544929855583
1 -0.968736676607071
};
\addlegendentry{Target}
\legend{}
\end{axis}

\end{tikzpicture}
    &\definecolor{fig_green}{RGB}{0,128,0}
    \begin{tikzpicture}
        \node at (0,0) {\begin{tikzpicture} 
    \begin{axis}[%
    hide axis,
    xmin=10,
    xmax=50,
    ymin=0,
    ymax=0.4,
    legend style={draw=white!15!black,legend cell align=left}
    ]
    \addlegendimage{semithick, fig_green}
    \addlegendentry{\tiny input};
    \addlegendimage{semithick, fig_green, dash pattern=on 5.55pt off 2.4pt}
    \addlegendentry{\tiny ground truth};
    \addlegendimage{semithick, blue};
    \addlegendentry{\tiny prediction};
    \end{axis}
\end{tikzpicture}};
    \node at (0,-1.) {~};
    \end{tikzpicture}
    \\
    \begin{tikzpicture}
            \node[rotate=90] at (0,0) {noisy};
            \node at (0,-1.65) {~};
    \end{tikzpicture}    
    &~ &\begin{tikzpicture}
\definecolor{darkgray176}{RGB}{176,176,176}
\definecolor{green}{RGB}{0,128,0}
\definecolor{lightgray204}{RGB}{204,204,204}

\begin{axis}[
    width=.3\textwidth,
    height=.3\textwidth/1.618,
legend cell align={left},
legend style={
  fill opacity=0.8,
  draw opacity=1,
  text opacity=1,
  at={(0.91,0.5)},
  anchor=east,
  draw=lightgray204
},
tick align=outside,
tick pos=left,
title={},
x grid style={darkgray176},
xlabel={\tiny time},
xmin=-1.13387096774194, xmax=1.10161290322581,
yticklabel style = {font = \tiny},
xticklabel style = {font = \tiny},
xtick style={color=black},
y grid style={darkgray176},
ylabel={\tiny $h(x_2(t))$},
ymin=-1.24434151103344, ymax=2.80816290983449,
ytick style={color=black}
]
\path [draw=blue, fill=blue, opacity=0.5]
(axis cs:0,-0.791884136999424)
--(axis cs:0,-0.791884136999424)
--(axis cs:0.0322580645161294,-0.532216532304019)
--(axis cs:0.064516129032258,0.207228559293209)
--(axis cs:0.096774193548387,1.12442474773327)
--(axis cs:0.129032258064516,1.68552099965311)
--(axis cs:0.161290322580645,1.56756057868288)
--(axis cs:0.193548387096774,0.908504125338844)
--(axis cs:0.225806451612903,0.146043870272565)
--(axis cs:0.258064516129032,-0.366882740438619)
--(axis cs:0.290322580645161,-0.596461580106275)
--(axis cs:0.32258064516129,-0.730623805831874)
--(axis cs:0.354838709677419,-0.913165271816812)
--(axis cs:0.387096774193548,-1.10654589884703)
--(axis cs:0.419354838709678,-1.19381854020165)
--(axis cs:0.451612903225806,-1.1546485659785)
--(axis cs:0.483870967741935,-1.08630398785724)
--(axis cs:0.516129032258065,-1.04849802441944)
--(axis cs:0.548387096774194,-0.935844580875931)
--(axis cs:0.580645161290323,-0.556390585828836)
--(axis cs:0.612903225806452,0.13874292125789)
--(axis cs:0.645161290322581,0.915902956607306)
--(axis cs:0.67741935483871,1.38730562268057)
--(axis cs:0.709677419354839,1.31274347655288)
--(axis cs:0.741935483870968,0.790708060720758)
--(axis cs:0.774193548387097,0.158464904288433)
--(axis cs:0.806451612903226,-0.298006792955262)
--(axis cs:0.838709677419355,-0.542627672001628)
--(axis cs:0.870967741935484,-0.721536980353704)
--(axis cs:0.903225806451613,-0.942388520552518)
--(axis cs:0.935483870967742,-1.15803749659711)
--(axis cs:0.967741935483871,-1.26309387894733)
--(axis cs:1,-1.24564270290988)
--(axis cs:1,-1.24564270290988)
--(axis cs:1,-0.553480274017001)
--(axis cs:0.967741935483871,-0.609085413130824)
--(axis cs:0.935483870967742,-0.531751521780701)
--(axis cs:0.903225806451613,-0.338120947030914)
--(axis cs:0.870967741935484,-0.137210070592655)
--(axis cs:0.838709677419355,0.0221297258593717)
--(axis cs:0.806451612903226,0.24789628860997)
--(axis cs:0.774193548387097,0.68638603957788)
--(axis cs:0.741935483870968,1.30113999406795)
--(axis cs:0.709677419354839,1.80656027208635)
--(axis cs:0.67741935483871,1.86440299114285)
--(axis cs:0.645161290322581,1.37557834692631)
--(axis cs:0.612903225806452,0.581761298262494)
--(axis cs:0.580645161290323,-0.128417100507951)
--(axis cs:0.548387096774194,-0.52080342500663)
--(axis cs:0.516129032258065,-0.644678468860948)
--(axis cs:0.483870967741935,-0.693848595375854)
--(axis cs:0.451612903225806,-0.773011108705527)
--(axis cs:0.419354838709678,-0.8207407389894)
--(axis cs:0.387096774193548,-0.740992616247295)
--(axis cs:0.354838709677419,-0.555241944757114)
--(axis cs:0.32258064516129,-0.380263917490026)
--(axis cs:0.290322580645161,-0.253998834224472)
--(axis cs:0.258064516129032,-0.0324742077203883)
--(axis cs:0.225806451612903,0.47322937545827)
--(axis cs:0.193548387096774,1.22944838870588)
--(axis cs:0.161290322580645,1.88216516482222)
--(axis cs:0.129032258064516,1.99311527479243)
--(axis cs:0.096774193548387,1.42551734614703)
--(axis cs:0.064516129032258,0.501904359793165)
--(axis cs:0.0322580645161294,-0.245188439871079)
--(axis cs:0,-0.509073434700109)
--cycle;

\addplot [semithick, blue]
table {%
0 -0.650478785849767
0.0322580645161294 -0.388702486087549
0.064516129032258 0.354566459543187
0.096774193548387 1.27497104694015
0.129032258064516 1.83931813722277
0.161290322580645 1.72486287175255
0.193548387096774 1.06897625702236
0.225806451612903 0.309636622865418
0.258064516129032 -0.199678474079504
0.290322580645161 -0.425230207165373
0.32258064516129 -0.55544386166095
0.354838709677419 -0.734203608286963
0.387096774193548 -0.923769257547163
0.419354838709678 -1.00727963959552
0.451612903225806 -0.963829837342014
0.483870967741935 -0.890076291616549
0.516129032258065 -0.846588246640194
0.548387096774194 -0.728324002941281
0.580645161290323 -0.342403843168393
0.612903225806452 0.360252109760192
0.645161290322581 1.14574065176681
0.67741935483871 1.62585430691171
0.709677419354839 1.55965187431961
0.741935483870968 1.04592402739435
0.774193548387097 0.422425471933156
0.806451612903226 -0.0250552521726459
0.838709677419355 -0.260248973071128
0.870967741935484 -0.42937352547318
0.903225806451613 -0.640254733791716
0.935483870967742 -0.844894509188906
0.967741935483871 -0.936089646039077
1 -0.899561488463438
};
\addlegendentry{Prediction}
\addplot [semithick, green]
table {%
-1.03225806451613 0.729314096916195
-1 1.5430002430531
-0.967741935483871 2.32204403870171
-0.935483870967742 1.831985798904
-0.903225806451613 0.851433703296098
-0.870967741935484 0.279354081877078
-0.838709677419355 -0.210590170919282
-0.806451612903226 -0.154089766822581
-0.774193548387097 -0.89795080770656
-0.741935483870968 -0.788817886111994
-0.709677419354839 -0.631217039088834
-0.67741935483871 -1.11565489128188
-0.645161290322581 -1.0186823441863
-0.612903225806452 -1.02455755343605
-0.580645161290323 -1.06131914172655
-0.548387096774193 -0.952046399419613
-0.516129032258065 -0.466246020590151
-0.483870967741935 0.262503198938305
-0.451612903225806 1.29173786310598
-0.419354838709677 2.51920556711464
-0.387096774193548 1.54711284138206
-0.354838709677419 0.915038117881071
-0.32258064516129 0.656212982887232
-0.290322580645161 0.221915856154588
-0.258064516129032 -0.281944627110773
-0.225806451612903 -0.633593952285014
-0.193548387096774 -0.895240317440411
-0.161290322580645 -1.00560588014814
-0.129032258064516 -0.974031437037782
-0.096774193548387 -0.866150101663904
-0.064516129032258 -1.07102836219983
-0.032258064516129 -0.835479545663183
};
\addlegendentry{Condition}
\addplot [semithick, green, dash pattern=on 5.55pt off 2.4pt]
table {%
0 -0.739654414051487
0.0322580645161294 -0.518329938112951
0.064516129032258 0.116995492630876
0.096774193548387 1.4830385503286
0.129032258064516 1.92918851810535
0.161290322580645 1.60484919087504
0.193548387096774 0.902948829211309
0.225806451612903 0.360083074779615
0.258064516129032 0.038407028127166
0.290322580645161 -0.535312914336627
0.32258064516129 -0.850668308618081
0.354838709677419 -0.773023157271037
0.387096774193548 -1.12963704153207
0.419354838709678 -1.06214760780795
0.451612903225806 -1.04132234674408
0.483870967741935 -0.728286752014149
0.516129032258065 -1.19049036317343
0.548387096774194 -0.786213082699808
0.580645161290323 -0.567579376473479
0.612903225806452 0.335513082155103
0.645161290322581 1.1374733999599
0.67741935483871 2.30787657939374
0.709677419354839 1.96813566944444
0.741935483870968 1.36715134974877
0.774193548387097 0.675192568868016
0.806451612903226 -0.179955450959619
0.838709677419355 -0.372618434067805
0.870967741935484 -0.210140874568701
0.903225806451613 -0.699218180799252
0.935483870967742 -0.996483428586153
0.967741935483871 -0.577931038414108
1 -1.19679478002127
};
\addlegendentry{Target}
\legend{}
\end{axis}

\end{tikzpicture}
    & \begin{tikzpicture}
\definecolor{darkgray176}{RGB}{176,176,176}
\definecolor{green}{RGB}{0,128,0}
\definecolor{lightgray204}{RGB}{204,204,204}

\begin{axis}[
    width=.3\textwidth,
    height=.3\textwidth/1.618,
legend cell align={left},
legend style={
  fill opacity=0.8,
  draw opacity=1,
  text opacity=1,
  at={(0.91,0.5)},
  anchor=east,
  draw=lightgray204
},
tick align=outside,
tick pos=left,
title={},
x grid style={darkgray176},
xlabel={\tiny time},
xmin=-1.13387096774194, xmax=1.10161290322581,
yticklabel style = {font = \tiny},
xticklabel style = {font = \tiny},
xtick style={color=black},
y grid style={darkgray176},
ymin=-1.24434151103344, ymax=2.80816290983449,
ytick style={color=black},
yticklabels={,,},
]
\path [draw=blue, fill=blue, opacity=0.5]
(axis cs:0,-1.665446045318)
--(axis cs:0,-1.665446045318)
--(axis cs:0.0322580645161294,-0.568897619418134)
--(axis cs:0.064516129032258,0.263031758801038)
--(axis cs:0.096774193548387,0.80837881476712)
--(axis cs:0.129032258064516,1.06689859178334)
--(axis cs:0.161290322580645,1.05987425552174)
--(axis cs:0.193548387096774,0.827975353710028)
--(axis cs:0.225806451612903,0.427391510412111)
--(axis cs:0.258064516129032,-0.0753940214809364)
--(axis cs:0.290322580645161,-0.609730884208127)
--(axis cs:0.32258064516129,-1.10730328548885)
--(axis cs:0.354838709677419,-1.50844223065865)
--(axis cs:0.387096774193548,-1.76767234362311)
--(axis cs:0.419354838709678,-1.85794397149951)
--(axis cs:0.451612903225806,-1.77307757287024)
--(axis cs:0.483870967741935,-1.52808132186236)
--(axis cs:0.516129032258065,-1.15719445331029)
--(axis cs:0.548387096774194,-0.709747383335965)
--(axis cs:0.580645161290323,-0.244192486820868)
--(axis cs:0.612903225806452,0.179086613560679)
--(axis cs:0.645161290322581,0.505374216416193)
--(axis cs:0.67741935483871,0.692700794731376)
--(axis cs:0.709677419354839,0.717475249796781)
--(axis cs:0.741935483870968,0.577886228992258)
--(axis cs:0.774193548387097,0.294555925632589)
--(axis cs:0.806451612903226,-0.0917064341730509)
--(axis cs:0.838709677419355,-0.52479970843623)
--(axis cs:0.870967741935484,-0.940245125827395)
--(axis cs:0.903225806451613,-1.27320560445854)
--(axis cs:0.935483870967742,-1.46655351593068)
--(axis cs:0.967741935483871,-1.47794982792781)
--(axis cs:1,-1.28506461143995)
--(axis cs:1,-1.28506461143995)
--(axis cs:1,-0.308494704515844)
--(axis cs:0.967741935483871,-0.502960017690905)
--(axis cs:0.935483870967742,-0.491970746399289)
--(axis cs:0.903225806451613,-0.298729000184287)
--(axis cs:0.870967741935484,0.03416669025934)
--(axis cs:0.838709677419355,0.449555573601588)
--(axis cs:0.806451612903226,0.882610046978934)
--(axis cs:0.774193548387097,1.26885004869444)
--(axis cs:0.741935483870968,1.55216499587092)
--(axis cs:0.709677419354839,1.69173937657283)
--(axis cs:0.67741935483871,1.66695089356827)
--(axis cs:0.645161290322581,1.47961336153512)
--(axis cs:0.612903225806452,1.15331907153329)
--(axis cs:0.580645161290323,0.730036686685325)
--(axis cs:0.548387096774194,0.264480380521523)
--(axis cs:0.516129032258065,-0.182967136242761)
--(axis cs:0.483870967741935,-0.553853726218446)
--(axis cs:0.451612903225806,-0.798849088776305)
--(axis cs:0.419354838709678,-0.883714042924335)
--(axis cs:0.387096774193548,-0.79343981256301)
--(axis cs:0.354838709677419,-0.53420442976094)
--(axis cs:0.32258064516129,-0.133056136115851)
--(axis cs:0.290322580645161,0.364529431938274)
--(axis cs:0.258064516129032,0.898881466603521)
--(axis cs:0.225806451612903,1.40168402321776)
--(axis cs:0.193548387096774,1.80229195788188)
--(axis cs:0.161290322580645,2.03423047083746)
--(axis cs:0.129032258064516,2.04131128161313)
--(axis cs:0.096774193548387,1.78285641240048)
--(axis cs:0.064516129032258,1.23761649151677)
--(axis cs:0.0322580645161294,0.406094571417332)
--(axis cs:0,-0.68887762798297)
--cycle;

\addplot [semithick, blue]
table {%
0 -1.17716183665049
0.0322580645161294 -0.0814015240004009
0.064516129032258 0.750324125158903
0.096774193548387 1.2956176135838
0.129032258064516 1.55410493669824
0.161290322580645 1.5470523631796
0.193548387096774 1.31513365579596
0.225806451612903 0.914537766814935
0.258064516129032 0.411743722561292
0.290322580645161 -0.122600726134926
0.32258064516129 -0.620179710802348
0.354838709677419 -1.02132333020979
0.387096774193548 -1.28055607809306
0.419354838709678 -1.37082900721192
0.451612903225806 -1.28596333082327
0.483870967741935 -1.0409675240404
0.516129032258065 -0.670080794776524
0.548387096774194 -0.222633501407221
0.580645161290323 0.242922099932229
0.612903225806452 0.666202842546984
0.645161290322581 0.992493788975658
0.67741935483871 1.17982584414982
0.709677419354839 1.2046073131848
0.741935483870968 1.06502561243159
0.774193548387097 0.781702987163514
0.806451612903226 0.395451806402942
0.838709677419355 -0.0376220674173207
0.870967741935484 -0.453039217784028
0.903225806451613 -0.785967302321414
0.935483870967742 -0.979262131164983
0.967741935483871 -0.990454922809359
1 -0.796779657977895
};
\addlegendentry{Prediction}
\addplot [semithick, green]
table {%
-1.03225806451613 0.729314096916195
-1 1.5430002430531
-0.967741935483871 2.32204403870171
-0.935483870967742 1.831985798904
-0.903225806451613 0.851433703296098
-0.870967741935484 0.279354081877078
-0.838709677419355 -0.210590170919282
-0.806451612903226 -0.154089766822581
-0.774193548387097 -0.89795080770656
-0.741935483870968 -0.788817886111994
-0.709677419354839 -0.631217039088834
-0.67741935483871 -1.11565489128188
-0.645161290322581 -1.0186823441863
-0.612903225806452 -1.02455755343605
-0.580645161290323 -1.06131914172655
-0.548387096774193 -0.952046399419613
-0.516129032258065 -0.466246020590151
-0.483870967741935 0.262503198938305
-0.451612903225806 1.29173786310598
-0.419354838709677 2.51920556711464
-0.387096774193548 1.54711284138206
-0.354838709677419 0.915038117881071
-0.32258064516129 0.656212982887232
-0.290322580645161 0.221915856154588
-0.258064516129032 -0.281944627110773
-0.225806451612903 -0.633593952285014
-0.193548387096774 -0.895240317440411
-0.161290322580645 -1.00560588014814
-0.129032258064516 -0.974031437037782
-0.096774193548387 -0.866150101663904
-0.064516129032258 -1.07102836219983
-0.032258064516129 -0.835479545663183
};
\addlegendentry{Condition}
\addplot [semithick, green, dash pattern=on 5.55pt off 2.4pt]
table {%
0 -0.739654414051487
0.0322580645161294 -0.518329938112951
0.064516129032258 0.116995492630876
0.096774193548387 1.4830385503286
0.129032258064516 1.92918851810535
0.161290322580645 1.60484919087504
0.193548387096774 0.902948829211309
0.225806451612903 0.360083074779615
0.258064516129032 0.038407028127166
0.290322580645161 -0.535312914336627
0.32258064516129 -0.850668308618081
0.354838709677419 -0.773023157271037
0.387096774193548 -1.12963704153207
0.419354838709678 -1.06214760780795
0.451612903225806 -1.04132234674408
0.483870967741935 -0.728286752014149
0.516129032258065 -1.19049036317343
0.548387096774194 -0.786213082699808
0.580645161290323 -0.567579376473479
0.612903225806452 0.335513082155103
0.645161290322581 1.1374733999599
0.67741935483871 2.30787657939374
0.709677419354839 1.96813566944444
0.741935483870968 1.36715134974877
0.774193548387097 0.675192568868016
0.806451612903226 -0.179955450959619
0.838709677419355 -0.372618434067805
0.870967741935484 -0.210140874568701
0.903225806451613 -0.699218180799252
0.935483870967742 -0.996483428586153
0.967741935483871 -0.577931038414108
1 -1.19679478002127
};
\addlegendentry{Target}
\legend{}
\end{axis}

\end{tikzpicture}
    &\begin{tikzpicture}
\definecolor{darkgray176}{RGB}{176,176,176}
\definecolor{green}{RGB}{0,128,0}
\definecolor{lightgray204}{RGB}{204,204,204}

\begin{axis}[
    width=.3\textwidth,
    height=.3\textwidth/1.618,
legend cell align={left},
legend style={
  fill opacity=0.8,
  draw opacity=1,
  text opacity=1,
  at={(0.91,0.5)},
  anchor=east,
  draw=lightgray204
},
tick align=outside,
tick pos=left,
title={},
x grid style={darkgray176},
xlabel={\tiny time},
xmin=-1.13387096774194, xmax=1.10161290322581,
yticklabel style = {font = \tiny},
xticklabel style = {font = \tiny},
xtick style={color=black},
y grid style={darkgray176},
ymin=-1.24434151103344, ymax=2.80816290983449,
ytick style={color=black},
yticklabels={,,},
]
\path [draw=blue, fill=blue, opacity=0.5]
(axis cs:0,-0.739898853705199)
--(axis cs:0,-0.739898853705199)
--(axis cs:0.0322580645161294,-0.442945945873865)
--(axis cs:0.064516129032258,0.177347417694429)
--(axis cs:0.096774193548387,1.1793753362096)
--(axis cs:0.129032258064516,1.9082301566808)
--(axis cs:0.161290322580645,1.74112809427527)
--(axis cs:0.193548387096774,1.0935983658069)
--(axis cs:0.225806451612903,0.475966804677645)
--(axis cs:0.258064516129032,0.00745292875602004)
--(axis cs:0.290322580645161,-0.330334796939457)
--(axis cs:0.32258064516129,-0.563439462459379)
--(axis cs:0.354838709677419,-0.705507882191618)
--(axis cs:0.387096774193548,-0.775254931041344)
--(axis cs:0.419354838709678,-0.789205189430896)
--(axis cs:0.451612903225806,-0.750487798384455)
--(axis cs:0.483870967741935,-0.651092336115223)
--(axis cs:0.516129032258065,-0.478710653656521)
--(axis cs:0.548387096774194,-0.227412542331311)
--(axis cs:0.580645161290323,0.083543012528116)
--(axis cs:0.612903225806452,0.395990468795268)
--(axis cs:0.645161290322581,0.630876528174991)
--(axis cs:0.67741935483871,0.730448613899942)
--(axis cs:0.709677419354839,0.685308374337714)
--(axis cs:0.741935483870968,0.528205912488009)
--(axis cs:0.774193548387097,0.309887456428381)
--(axis cs:0.806451612903226,0.0779767485881812)
--(axis cs:0.838709677419355,-0.132959994201391)
--(axis cs:0.870967741935484,-0.302395337257076)
--(axis cs:0.903225806451613,-0.4204759409994)
--(axis cs:0.935483870967742,-0.483716978187525)
--(axis cs:0.967741935483871,-0.491757964036874)
--(axis cs:1,-0.446135940069033)
--(axis cs:1,-0.446135940069033)
--(axis cs:1,-0.446135940069033)
--(axis cs:0.967741935483871,-0.491757964036874)
--(axis cs:0.935483870967742,-0.483716978187525)
--(axis cs:0.903225806451613,-0.4204759409994)
--(axis cs:0.870967741935484,-0.302395337257076)
--(axis cs:0.838709677419355,-0.132959994201391)
--(axis cs:0.806451612903226,0.0779767485881812)
--(axis cs:0.774193548387097,0.309887456428381)
--(axis cs:0.741935483870968,0.528205912488009)
--(axis cs:0.709677419354839,0.685308374337714)
--(axis cs:0.67741935483871,0.730448613899942)
--(axis cs:0.645161290322581,0.630876528174991)
--(axis cs:0.612903225806452,0.395990468795268)
--(axis cs:0.580645161290323,0.083543012528116)
--(axis cs:0.548387096774194,-0.227412542331311)
--(axis cs:0.516129032258065,-0.478710653656521)
--(axis cs:0.483870967741935,-0.651092336115223)
--(axis cs:0.451612903225806,-0.750487798384455)
--(axis cs:0.419354838709678,-0.789205189430896)
--(axis cs:0.387096774193548,-0.775254931041344)
--(axis cs:0.354838709677419,-0.705507882191618)
--(axis cs:0.32258064516129,-0.563439462459379)
--(axis cs:0.290322580645161,-0.330334796939457)
--(axis cs:0.258064516129032,0.00745292875602004)
--(axis cs:0.225806451612903,0.475966804677645)
--(axis cs:0.193548387096774,1.0935983658069)
--(axis cs:0.161290322580645,1.74112809427527)
--(axis cs:0.129032258064516,1.9082301566808)
--(axis cs:0.096774193548387,1.1793753362096)
--(axis cs:0.064516129032258,0.177347417694429)
--(axis cs:0.0322580645161294,-0.442945945873865)
--(axis cs:0,-0.739898853705199)
--cycle;

\addplot [semithick, blue]
table {%
0 -0.739898853705199
0.0322580645161294 -0.442945945873865
0.064516129032258 0.177347417694429
0.096774193548387 1.1793753362096
0.129032258064516 1.9082301566808
0.161290322580645 1.74112809427527
0.193548387096774 1.0935983658069
0.225806451612903 0.475966804677645
0.258064516129032 0.00745292875602004
0.290322580645161 -0.330334796939457
0.32258064516129 -0.563439462459379
0.354838709677419 -0.705507882191618
0.387096774193548 -0.775254931041344
0.419354838709678 -0.789205189430896
0.451612903225806 -0.750487798384455
0.483870967741935 -0.651092336115223
0.516129032258065 -0.478710653656521
0.548387096774194 -0.227412542331311
0.580645161290323 0.083543012528116
0.612903225806452 0.395990468795268
0.645161290322581 0.630876528174991
0.67741935483871 0.730448613899942
0.709677419354839 0.685308374337714
0.741935483870968 0.528205912488009
0.774193548387097 0.309887456428381
0.806451612903226 0.0779767485881812
0.838709677419355 -0.132959994201391
0.870967741935484 -0.302395337257076
0.903225806451613 -0.4204759409994
0.935483870967742 -0.483716978187525
0.967741935483871 -0.491757964036874
1 -0.446135940069033
};
\addlegendentry{Prediction}
\addplot [semithick, green]
table {%
-1.03225806451613 0.729314096916195
-1 1.5430002430531
-0.967741935483871 2.32204403870171
-0.935483870967742 1.831985798904
-0.903225806451613 0.851433703296098
-0.870967741935484 0.279354081877078
-0.838709677419355 -0.210590170919282
-0.806451612903226 -0.154089766822581
-0.774193548387097 -0.89795080770656
-0.741935483870968 -0.788817886111994
-0.709677419354839 -0.631217039088834
-0.67741935483871 -1.11565489128188
-0.645161290322581 -1.0186823441863
-0.612903225806452 -1.02455755343605
-0.580645161290323 -1.06131914172655
-0.548387096774193 -0.952046399419613
-0.516129032258065 -0.466246020590151
-0.483870967741935 0.262503198938305
-0.451612903225806 1.29173786310598
-0.419354838709677 2.51920556711464
-0.387096774193548 1.54711284138206
-0.354838709677419 0.915038117881071
-0.32258064516129 0.656212982887232
-0.290322580645161 0.221915856154588
-0.258064516129032 -0.281944627110773
-0.225806451612903 -0.633593952285014
-0.193548387096774 -0.895240317440411
-0.161290322580645 -1.00560588014814
-0.129032258064516 -0.974031437037782
-0.096774193548387 -0.866150101663904
-0.064516129032258 -1.07102836219983
-0.032258064516129 -0.835479545663183
};
\addlegendentry{Condition}
\addplot [semithick, green, dash pattern=on 5.55pt off 2.4pt]
table {%
0 -0.739654414051487
0.0322580645161294 -0.518329938112951
0.064516129032258 0.116995492630876
0.096774193548387 1.4830385503286
0.129032258064516 1.92918851810535
0.161290322580645 1.60484919087504
0.193548387096774 0.902948829211309
0.225806451612903 0.360083074779615
0.258064516129032 0.038407028127166
0.290322580645161 -0.535312914336627
0.32258064516129 -0.850668308618081
0.354838709677419 -0.773023157271037
0.387096774193548 -1.12963704153207
0.419354838709678 -1.06214760780795
0.451612903225806 -1.04132234674408
0.483870967741935 -0.728286752014149
0.516129032258065 -1.19049036317343
0.548387096774194 -0.786213082699808
0.580645161290323 -0.567579376473479
0.612903225806452 0.335513082155103
0.645161290322581 1.1374733999599
0.67741935483871 2.30787657939374
0.709677419354839 1.96813566944444
0.741935483870968 1.36715134974877
0.774193548387097 0.675192568868016
0.806451612903226 -0.179955450959619
0.838709677419355 -0.372618434067805
0.870967741935484 -0.210140874568701
0.903225806451613 -0.699218180799252
0.935483870967742 -0.996483428586153
0.967741935483871 -0.577931038414108
1 -1.19679478002127
};
\addlegendentry{Target}
\legend{}
\end{axis}

\end{tikzpicture}
    &
    \end{tabular}
    
     \vspace{-0.7\intextsep}
\caption{Multi-step mean and 2-sigma interval of the prediction for predator population from the predator-prey dynamics for our proposed Koopman-equivariant GP (KE-GP), a generic contextual kernel (C-GP), and a Koopman operator regression approach (KOR) for noise-free (top) and noisy (bottom) training data.}
\label{fig:illustrative}
\end{figure*}

\label{section:svigps}
GP model scale poorly with the dataset size, requiring $\BigO(N^3)$ computations and $\BigO(N^2)$ memory during training. To address this, in the following we present a sparse GP approximation that uses a variational inference approach. Our approach closely follows stochastic variational inference with sparse GPs \citep{hensmann2013gaussian,Wilk2018}, with some additional modifications to the selection and optimization of inducing points. As discussed at the end of this section, this choice allows considerable scalability during training \ref{itm:Scale}, and presents desirable properties when used in conjunction with our equivariant covariance function.
\subsection{Variational Inference with Sparse GPs}

	During training, computational complexity stems mainly from the inversion of the data covariance matrix $\Kff$, where $\left[\Kff\right]_{nn'} =  k^{\txt{KE}}_y((t,\vz^{(n)}),(t',\vz^{(n^\prime)})) =: k^{\txt{KE}}_{y(t,t^\prime)}(\vz^{(n)},\vz^{(n^\prime)})$. To address these problems, we resort to variational inference using \emph{inducing variables}  \citep{quinonero2005unifying,hensmann2013gaussian}. %
 We obtain the sparse GP by considering $M \ll N$ inducing observations $\vm$, corresponding to the inducing trajectories $\smash{\{{\vz}^{(m)}:=\bm{x}^{(m)}_{[\tau_s,\tau_e]}\}_{m=1}^M = \mZ}$. Instead of employing the GP prior for the trajectories $\vm$ corresponding to $\mZ$, we place a simpler Gaussian prior $q(\cdot)$ over $\vm$, specified by a mean $\vm$ and covariance $\mS$. By leveraging a variational inference argument \citep{hensmann2013gaussian}, we then obtain the approximate Gaussian process posterior %
 $\GP(\tilde{\mu}(\cdot), \tilde{\sigma}^2(\cdot, \cdot))$ with mean and variance
	\begin{align}
 \begin{split}
	\tilde{\mu}(\cdot) {=} & \Kdu\Kuu\inv\vm,\qquad \\  \tilde{\sigma}^2(\cdot, \cdot) {=}& \kdd - \Kdu\Kuu\inv[\Kuu{ -} \mS]\Kuu\inv\Kud \label{eq:varpost}\,
 \end{split}
	\end{align}
	where $[\Kuu]_{ij} = k^{\txt{KE}}_{y(t,t^\prime)}(\vz^{(i)}, \vz^{(j)})$ %
 and $\smash{\Kud = [k^{\txt{KE}}_{y(t,t^\prime)}(\vz^{(m)}, \cdot)]_{m=1}^M}$. The shape of the posterior can be adjusted by changing the values ${\mZ}$ and output mean $\vm$ and variance $\mS$ of the inducing outputs. Here, we follow an approach similar to \cite{hensmann2013gaussian}, which allows us to minimize by sampling batches of data instead of computing the full gradient, improving memory complexity to $\mathcal{O}(BM+M^2)$. We choose the hyperparameters and inducing points jointly by minimizing the loss
 \begin{align}
 \begin{split}
     \textstyle\sum_{i=1}^{N} &\big( %
 {-}\textstyle \frac{1}{2}\log(2\pi\sigma_{\txt{on}}^{2}) {-}\frac{\sigma_{\txt{on}}^{2}}{2}(y_i-\bm{k}_i^\top \Kuu^{-1} \bm{m})^2 
\\ &{-} \textstyle \frac{1}{2} \sigma_{\txt{on}}^2 \tilde{k}_{i,i} {-} \frac{1}{2} \text{tr} \left( \bm{S} \bm{\Lambda}_i \right) \big)
- \text{KL} \left( q(\bm{u}) \| p(\bm{u}) \right),
\end{split}
 \end{align}
where $\bm{k}_i =[k^{\txt{KE}}_{y(t,t^\prime)}({\vz}^{(m)}, \bm{x}^{(i)}_{[\tau_s,\tau_e]})]_{m=1}^M$, $\Lambda_i =  \Kuu^{-1} \bm{k}_i \bm{k}_i^\top \Kuu^{-1}
$, $\tilde{k}_{i,i} = [\Kff - \Kfu \Kuu^{-1} \Kfu]_{ii}$, and $[\Kfu]_{ij} = k^{\txt{KE}}_{y(t,t^\prime)}(\bm{x}^{(i)}_{[\tau_s,\tau_e]},\vz^{(j)})$. However, unlike \cite{hensmann2013gaussian}, we only optimize inducing trajectories and avoid sampling any time/context-related inducing points, which is due to the structure of the spectral decomposition \eqref{eq:KoopObs}. This allows a significant reduction in training complexity compared, e.g., to the generic contextual kernel $k^{\txt{C}}(\cdot,\cdot):=k^{\txt{SE}}(t,t^\prime)\otimes k^{\txt{SE}}(\bm{x}_{0},\bm{x}^\prime_{0})$, 
where inducing points representing time are also optimized.

Due to the structure of our Koopman-equivariant construction, and the resulting benefits in information gain presented in Section \ref{section:analysisofsamplecomplexity}, our approach is also more robust to a lack of correlation between points, an issue commonly observed in conventional sparse GP approximations \citep{Murray2010,HensmanNIPS2015}. In particular, a lower information gain implies that less inducing points are required than with conventional GPs to accurately represent the full posterior \citep{Burt2019RatesRegression}.%

\section{NUMERICAL EXPERIMENTS}\label{sec:NumExp}
To demonstrate the applicability of KE-GPs to realistic data, we perform qualitative and quantitative studies on a set of benchmark examples. As a classical dynamical systems example, we choose the predator-prey model; from the robotics domain, we consider expert demonstrations on the halfcheetah environment from D4RL~\citep{fu2020d4rl} and forecast the first state and action; as a high uncertainty example we choose temperature data from the Monash TSF benchmark \citep{godahewa2021monash} taken at \textit{Oikolab} -- demonstrating the usefulness of building in \eqref{eq:KoopObs} structure as a prior for highly complex weather dynamics. Since these datasets provide a single long trajectory, we split off the last chunk as test data and partition the trajectory into $N$ input-task pairs to comply with our model structure.\looseness=-1

\begin{table}[t!]%
    \caption{Comparable nonparametric frameworks.}
    \label{tab:contributionNonparam}
    \centering
    \footnotesize
    \vspace{-2ex}
    \begin{tabular}{l|ccc}
    \toprule
        Method& \ref{itm:Trac} & \ref{itm:Eff} & \ref{itm:Scale}\\
        \midrule
        C-GP \citep{Li2024STkernel} & ({\color{green!80!black}{\text{\cmark}}})   & \color{red!80!black}{\text{\xmark}} & \color{green!80!black}{\text{\cmark}}\\
        KOR \citep{Kostic2022LearningSpaces} & \color{red!80!black}{\text{\xmark}}  & ({\color{green!80!black}{\text{\cmark}}}) & \color{green!80!black}{\text{\cmark}} \\
        {\text{\textbf{KE-GP} (ours)}} & \color{green!80!black}{\text{\cmark}} & \color{green!80!black}{\text{\cmark}} & \color{green!80!black}{\text{\cmark}} \\
        \bottomrule
    \end{tabular}
\end{table}
\begin{table}[t!]%
    \caption{Simulations on small subsets of the Predator-Prey (PP), D4RL Half-Cheetah (D4RL), and Oikolab Temparature (OT) datasets. We report RMSE in mean and standard deviation for 5 runs. Training data are $N$ past trajectories over a unit-normalized interval, discretized using $H$ equidistant points.}
    \label{tab:smallsim}
    \centering
    \footnotesize
    \vspace{-2ex}
    \begin{tabular}{lc|ccc}
    \toprule
      & $N\times H$& \tbm{KE-GP} & C-GP & KOR\\
   \midrule
    PP&32$\times$32 & 0.28$\pm$0.0& 0.60$\pm$0.0 & \textbf{0.27}$\pm$0.0\\
    D4RL&32$\times$16 & 0.46$\pm$0.0& 0.98$\pm$0.0 & \textbf{0.44}$\pm$0.0\\
    OT&32$\times$16 & \textbf{0.63}$\pm$0.0& 0.68$\pm$0.0 & 0.86$\pm$0.0\\
    
        \bottomrule
    \end{tabular}
\end{table}
\begin{table}[t!]%
    \caption{Simulations on large subsets of the Predator-Prey (PP), D4RL Half-Cheetah (D4RL), and Oikolab Temparature (OL) datasets. Training data are $N$ past trajectories over a unit-normalized interval, discretized using $H$ equidistant points.}
    \label{tab:largesim}
    \centering
    \footnotesize
    \vspace{-2ex}
    \begin{tabular}{lc|ccc}
    \toprule
      & $N\times H$& \tbm{KE-GP} & C-GP & KOR\\
   \midrule
    PP &\phantom{0}512$\times$32& \textbf{0.26}$\pm$0.0\phantom{0}& 0.42$\pm$0.0\phantom{0} & 0.53$\pm$0.0\\
    D4RL$\!\!\!$ &3000$\times$16& 0.48$\pm$0.02 & 0.66$\pm$0.07& \textbf{0.44}$\pm$0.0\\
    OT &4000$\times$16& \textbf{0.54}$\pm$0.03& 0.60$\pm$0.02 & 0.71$\pm$0.0\\
    
        \bottomrule
    \end{tabular}
\end{table}

\textbf{Baselines~~} To put our novel algorithm into perspective, we compare to two standard approaches: Gaussian Processes with the time-dependent context (C-GP) by \cite{Li2024STkernel} and operator regression for dynamical systems (KOR)~\citep{Kostic2022LearningSpaces} from the \hyperlink{https://github.com/Machine-Learning-Dynamical-Systems/kooplearn}{\txt{kooplearn}} package and equip it with SciPy's \citep{2020SciPy-NMeth} \textit{minimize} for hyperparameter tuning. While Koopman-based, their method does not forecast via a decomposable model akin to \eqref{eq:KoopObs}, but requires taking powers of a $D{\times}D$ dense matrix. 
As summarized in Table~\ref{tab:contributionNonparam}, these two methods exhibit some of the important properties discussed in Section \ref{sec:ProbStat}, such that they are valuable baselines.\looseness=-1

\begin{figure}[t]
    \centering
    \begin{tikzpicture}

\definecolor{darkgray176}{RGB}{176,176,176}
\definecolor{darkorange25512714}{RGB}{255,127,14}
\definecolor{forestgreen4416044}{RGB}{44,160,44}
\definecolor{lightgray204}{RGB}{204,204,204}
\definecolor{steelblue31119180}{RGB}{31,119,180}

\begin{axis}[
width=5.7cm, height=3.5cm,
legend cell align={left},
legend style={
  at={(1.02,0.8)},
  anchor=north west,
},
tick align=outside,
tick pos=left,
x grid style={darkgray176},
xmin=15, xmax=10000,
xmode=log,
xlabel = {\text{\tiny \# datapoints}},
ylabel = {\text{\tiny emp. info. gain}},
xtick style={color=black},
xtick={0.01,0.1,1,10,100,1000,10000,100000,1000000},
xticklabels={
  \(\displaystyle {10^{-2}}\),
  \(\displaystyle {10^{-1}}\),
  \(\displaystyle {10^{0}}\),
  \(\displaystyle {10^{1}}\),
  \(\displaystyle {10^{2}}\),
  \(\displaystyle {10^{3}}\),
  \(\displaystyle {10^{4}}\),
  \(\displaystyle {10^{5}}\),
  \(\displaystyle {10^{6}}\)
},
y grid style={darkgray176},
ymin=5, ymax=4000,
ymode=log,
ytick style={color=black},
ytick={0.01,0.1,1,10,100,1000,10000,100000},
yticklabels={
  \(\displaystyle {10^{-2}}\),
  \(\displaystyle {10^{-1}}\),
  \(\displaystyle {10^{0}}\),
  \(\displaystyle {10^{1}}\),
  \(\displaystyle {10^{2}}\),
  \(\displaystyle {10^{3}}\),
  \(\displaystyle {10^{4}}\),
  \(\displaystyle {10^{5}}\)
},
yticklabel style = {font = \tiny},
xticklabel style = {font = \tiny},
]
\addplot [very thick, steelblue31119180,]
table {%
9999 2264.11377853996
7278 1804.26327966419
5298 1420.04472403012
3856 1095.40602083438
2807 836.680567640406
2043 631.775979654692
1487 472.562066757352
1082 352.135588183581
788 260.479692402163
573 191.711486161951
417 140.513623368777
303 102.644311381288
221 75.2133200772773
161 55.159666675983
117 40.2831954700074
85 29.31805807427
62 21.4780064745164
45 15.586340006635
32 11.0902089995822
23 7.97107383279624
17 5.89163232362588
12 4.15888308335967
9 3.11916231251975
6 2.07944154167984
4 1.38629436111989
3 1.03972077083992
2 0.693147180559945
1 0.346573590279973
1 0.346573590279973
1 0.346573590279973
};
\addlegendentry{\tiny \ref{eq:SDK}$(\bm{x}_0)$}
\addplot [very thick, forestgreen4416044, ]
table {%
9999 584.656726983284
7278 482.648758569247
5298 400.018580700851
3856 327.016416026596
2807 267.936837961897
2043 227.51171555213
1487 176.282258165846
1082 127.923214142905
788 116.989659625437
573 91.4829039943821
417 77.3576902555283
303 54.7094047705121
221 41.356397075653
161 45.5260312182046
117 25.0994743235434
85 22.1324167310142
62 16.1900860183269
45 12.9863449371233
32 9.59705643999886
23 7.41139449351921
17 5.41375041363105
12 3.95743297067132
9 3.10929738877
6 2.03296155184543
4 1.38344518620096
3 0.987590703724384
2 0.687289146233464
1 0.346573590279973
1 0.346573590279973
1 0.346573590279973
};
\addlegendentry{\tiny \textbf{\ref{eq:KE-SDK}} (true)}

\addplot [very thick, forestgreen4416044, dashed]
table {%
9999 1150.24144167382
7278 921.876964566388
5298 730.129320722141
3856 569.321298285514
2807 444.464772609148
2043 338.563956158658
1487 256.820764550324
1082 199.371251263352
788 151.518582913528
573 128.056611743859
417 88.4672878074439
303 70.909756169937
221 55.4035601061034
161 34.2302972124705
117 33.418917895361
85 24.5881170687663
62 18.9084220075427
45 14.4666214371714
32 10.138627472861
23 7.38739589901983
17 5.79102868865375
12 4.13377361608335
9 3.02188727225571
6 1.93259562627923
4 1.38548323020342
3 1.00804375391702
2 0.69309479616628
1 0.346573590279973
1 0.346573590279973
1 0.346573590279973
};
\addlegendentry{\tiny \textbf{\ref{eq:KE-SDK}} (random)$\!\!\!$}

\end{axis}

\end{tikzpicture}
    \vspace{-1.5\intextsep}
    \caption{Empirical information gain $\hat{\gamma}$ for a 2D linear system scaled to remove effects of constants. %
    The improved rates confirm our theoretical results for Koopman-equivariant GPs, leading to a lower information gain compared to their non-equivariant counterpart \eqref{eq:SDK}, even when a randomly sampled eigenvalue spectrum $\{\eig_j\}_{j=1}^D$  is used instead of the true spectrum.
    }
    \label{fig:enter-label}
\end{figure}
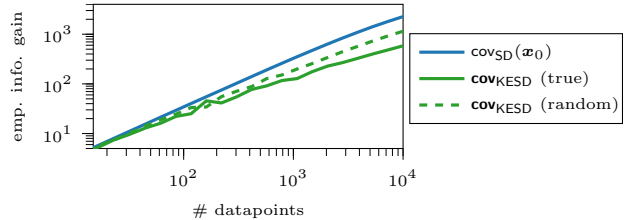
\textbf{Qualitative Comparison~~} We first qualitatively compare the different approaches on the predator-pray model as illustrated in Figure~\ref{fig:illustrative}. It can be clearly seen that all methods allow to accurately predict the future trajectory when given noise-free data from the dynamical system. However, when the state trajectories are perturbed by noise as commonly encountered in practice, significant differences between the predictions become apparent. While our proposed KE-GP maintains a high accuracy and reasonably small confidence intervals, the estimated uncertainty of the C-GP considerably grows and the prediction accuracy for longer horizons significantly drops for the Koopman operator regression (KOR) approach from \cite{Kostic2022LearningSpaces}. This high accuracy of KE-GPs can be attributed to their strong generalization capabilities captured by the information gain as discussed in Section \ref{section:analysisofsamplecomplexity}. When empirically comparing this value, we can immediately see an improvement over non-equivariant kernels, cf. Figure \ref{fig:enter-label}.\looseness=-1

\textbf{Quantitative Evaluation~~} We perform two evaluations for each model run and dataset: a small subset for which exact inference is possible and a large subset handled using variational inference. 
We observe that KE-GP performs robustly on all datasets and sizes as depicted in Tables~\ref{tab:smallsim} and \ref{tab:largesim}. While it consistently outperforms the C-GP, the KOR baseline is better on some datasets but the difference in accuracy is marginal. Importantly, KE-GPs provide a significant improvement over KOR for the other data sets. This is fully in line with our qualitative comparison, which shows that KOR can be sensitive to noise with a severe impact on its performance. In addition, we want to stress here that the KOR method does not come with methods for automated model selection, such that manual parameter tuning was necessary to make it competitive. Therefore, this comparison clearly demonstrates the improved generalization ability achieved by embedding the operator-theoretic foundations in our KE-GP approach.\looseness=-1

\section{CONCLUSION}
\label{sec:Concl}
We presented a novel approach to incorporate an operator-theoretic dynamical system structure into Gaussian process regression. Our framework enables a tractable probabilistic treatment of continuous-time dynamical models not present in existing literature. Utilizing a symmetrization tailored to dynamical systems, based on the concept of Koopman-equivariance (KE), we achieve a sample-complexity reduction compared to a contextual kernel without our proposed \eqref{eq:KoopObs} structure. In scaling to large datasets we exploit our model structure to avoid sampling any time/context-related inducing points. Hence, it does not suffer from a lack of correlation between inducing points, which is common for conventional sparse GP.
Through numerical experiments, we show the utility of our KE-GP, demonstrating superior prediction performance to vanilla contextual GPs and on par or better than Koopman operator learning.

\subsubsection*{Acknowledgements}
This work is supported by the DAAD programme Konrad Zuse Schools of Excellence in Artificial Intelligence (relAI), sponsored by the German Federal Ministry of Education and Research and by the European Research Council (ERC) Consolidator Grant “Safe data-driven control for human-centric systems (CO-MAN)” under grant agreement number 864686.

 % !TEX root = main.tex
\onecolumn
\aistatstitle{
Supplementary Materials for ``Koopman-Equivariant Gaussian Processes"}

\appendix

The supplementary materials are organized as follows.
\begin{itemize}
\item Appendix~\ref{supl:related} expands on various aspects of related work from the main paper in greater detail.
    \item  Appendix~\ref{supl:koop} contains background on standing assumptions and spectral theory of Koopman operators. 
    \item To complement the sample complexity analysis, we address our framework's representation/approximation power in Appendix \ref{supl:repPWR}.
    \item  The proofs of theoretical results are found in Appendix~\ref{supl:Proof}.
    \item Details on the setup of numerical experiments are found in Appendix \ref{supl:NumExp}.
\item Finally, Appendix~\ref{supl:AddExp} includes additional experiments and ablation studies.
\end{itemize}

\section{EXPANDED RELATED WORK}\label{supl:related}
\paragraph{Vanilla GPs}
Gaussian process (GP) regression \citep{Rasmussen2006} has attracted attention for learning nonlinear dynamical systems due to its capability of inferring models with little structural prior knowledge: either by using so-called universal kernels \citep{Micchelli2006UniversalKernels} or placing a prior on a set of kernels and optimizing their likelihood of explaining the data \citep{Duvenaud2014}. 
In particular, their ability to quantify epistemic uncertainty has led to a common application in safety-critical control problems \citep{BerkenkampECC15,pmlr-v37-sui15,NIPS2017_766ebcd5,CuriCDC22}. However, commonly used GP models are single-step predictors, such that uncertainty propagation necessitates approximations when predicting probability distributions more than a single time step in the future. Uncertainty propagation often relies on iterative approaches, in which the previous predictions are used as uncertain inputs to the GP model. This can be exploited in a sampling-based fashion by randomly drawing states \citep{Bradford2019} or using the unscented transform \citep{Ko2007a}.
While the computational complexity of sampling-based approaches can be reduced through further approximations \citep{pmlr-v120-hewing20a, TBpredGP}, it generally remains high.
Approximating the predictive distributions, e.g., using a Taylor approximation \citep{Girard2003} or through exact moment matching \citep{Deisenroth2011PILCO:Search} can reduce the complexity, there are no accuracy guarantees of these approximations for long-term forecasts. 
Direct solutions to these challenges include, e.g., direct modeling of uncertainty intervals \citep{Polymenakos2020, Curi2020} or using ``a GP per time-step'' of prediction \citep{Pfefferkorn2022}, which suffers from a lack of time-correlation and forecast non-linearity.

\paragraph{State-space GPs}
A variety of works considers models with task correlation. Considering modeling dynamical systems, latent variable state-space models have the ability to decouple the model into the dynamics (process) and static (output) structures \citep{Wang2005_ccd45007,pmlr-v9-titsias10a, Frigola2014, Damianou2016, NIPS2017_1006ff12}. Like the recent work of \citep{Fan2023}, these models are limited to settings where a single trajectory is available and do not exploit any time-series structure. Still, they require posterior approximations due to the nonlinearity of latent dynamics. 
Aiding tractability, some works consider linear time-invariant (LTI) models \citep{pmlr-v5-alvarez09a,Sarkka2013}, but come with strong prior-knowledge requirements and unclear representational power. In particular, our KE-GPs can be considered as a continuous
contextual GP for dynamical systems with \eqref{eq:KoopObs} structure.

\paragraph{Koopman operator-based approaches}
While operator regression \citep{Williams2015ADecomposition,Klus2020EigendecompositionsSpaces,Kostic2022LearningSpaces,Li2022OptimalLearning,Ishikawa2024,Mauroy2024,Meunier2024} could be applied to build an LTI predictor \eqref{eq:KoopObs}, this comes with inherent limitations. Namely, the recovery of normal spectra and eigenspaces of $\mathcal{A}_t$ using operator regression in an infinite-dimensional RKHS is an \textit{ill-posed inverse problem} \citep{Knapik2011,Knapik2016,Horowitz2014}. Spectral estimation gets increasingly hard with the eigenvalue decay of the covariance \citep{klebanov2020rigorous}, limiting the utility of estimated spectra and eigenspaces. To mitigate these effects, \citet{Kostic2023KoopmanEigenvalues} suggests using low-rank estimators and empirically estimated RKHSs \citep{kostic2024learning} to control the degree of ill-posedness. However, there is no guarantee such a low-rank representation would span an observable of interest and form an LTI predictor \eqref{eq:KoopObs}. In stark contrast, our KE-GP regression bypasses this ill-posedness by construction and directly learns a universal representation of \eqref{eq:KoopObs} in a probabilistic fashion using Bayesian principles.
    Furthermore, \textit{our approximation-based complexity bounds} (in terms of information gain) \textit{are measure-independent and do not require any i.i.d.-type sampling assumptions}. This is in stark contrast to state-of-the-art concentration results in Koopman operator learning by \cite{Kostic2022LearningSpaces,Kostic2023KoopmanEigenvalues} that are dependent on measures, cf. \citep{pmlr-v75-belkin18a} for a discussion.

\paragraph{Works connecting Koopmanism and GPs}
Previous attempts at connecting Koopmanism to GPs \citep{Lian2020OnOperators,10.1162/neco_a_01555,Loya2023}  rely on heuristics and ad-hoc chocies, lacking theoretical justification as well as rigorous representational considerations. Furthermore, they hinge on heuristics by applying subspace-identification or dynamic mode decomposition before applying Gaussian process regression. In contrast, we offer a principled and fully-tractable approach with provable representational and learning guarantees.

\paragraph{Signature kernels}
Also geared towards sequential data, there is a recent rise in popularity of so-called \textit{signature kernels} \citep{Kiraly2016KernelsData, Lee2023TheKernel,Salvi2021ThePDE, Lemercier2021SigGPDE:Data} that also use a symmetrization to be rendered time-reparametrization invariant. This prohibits the extraction of dynamical system representations related to transfer operators and their eigenfunctions. Crucially, time-reparametrization-invariance allows them to excel at discriminative tasks \citep{Lemercier2021SigGPDE:Data,Salvi2021ThePDE} but not at generative tasks such as long-term forecasting \citep{KKR_neurips2023}.

\section{KOOPMAN OPERATOR MODELS FOR DETERMINISTIC DYNAMICS}\label{supl:koop}

\begin{remark}[Operator boundedness]\label{rmk:bounded}
    Consider a forward complete system on a compact set $\Set{X}$ and a continuous flow $\bm{F}_{t}$. It is well-known that a  time-$t$ Koopman operator $\mathcal{A}_t$ is then a contraction semigroup on ${C}(\Set{X})$ \citep{Kreidler2018CompactSystems}. Due to forward completeness of the flow, we therefore obtain a Banach algebra ${C}(\Set{X})$ with a bounded semigroup $\{\mathcal{A}_t\}_{t\geq0} \in \mathcal{B}({C}(\Set{X}))$.%
\end{remark}

\begin{definition}[Non-recurrent domain]\label{def:Nonrec}
Let time \(T \in(0, \infty)\) be given. A set \(\Set{X}_0\subset\) \(\Set{X}\) is called nonrecurrent if
\[
\bm{x}\in \Set{X}_0 \Longrightarrow \bm{F}_t(\bm{x}) \notin \Set{X}_0 \quad \forall t \in(0, T].
\]
A non-recurrent domain is the image $\Set{X}_{T}$ of non-recurrent set of initial conditions $\Set{X}_0$ traced out by the flow map $\bm{F}_t(\cdot)$
\[
\Set{X}_{T}=\bigcup_{t \in[0, T]} \bm{F}_t(\Set{X}_0)=\bigcup_{t \in[0, T]}\left\{\bm{F}_t\left(\bm{x}_{0}\right) \mid \bm{x}_{0} \in\Set{X}_0\right\}.
\]
\end{definition}

Less formally, one can think of the non-recurrent domain as the domain \emph{where flow does not intersect itself}.

Practically, non-recurrence is commonly ensured by a choice of the time interval $[0,T]$ so no periodicity is exhibited. Note that it does not mean the system's behavior is not allowed the be periodic, but our perception of it via data does. Effectively this prohibits the multi-valuedness of eigenfunctions -- allowing them to define an injective feature map.
    Thus, non-recurrence is a certain but general condition that bounds the time-horizon $T$ in which it is feasible to completely describe the nonlinear system's flow via an LTI predictor.
    
Note that our Assumption \ref{asm:Sspec} requires the existence of a nonrecurrent set that allows for a nonrecurernt domain.
It makes for a less-restrictive and intuitive condition compared to existing RKHS approaches \citep{Kostic2022LearningSpaces,Kostic2023KoopmanEigenvalues} that rely on the self-adjointness and compactness of the actual Koopman operator, which is rarely fulfilled for deterministic dynamics \eqref{eq:SSmodel} and hard to verify without prior knowledge.

\subsection{Koopman Mode Decomposition (KMD)}

As in the main text, when referring to \textit{Koopman Mode Decomposition} \eqref{eq:KoopObs}, we let the eigenfunctions absorb the spatial mode coefficients $\langle{ g^\prime_j,h}\rangle$ (possible w.l.o.g.) as they correspond to eigenfunctions $g_j$ and not eigenvalues $\eig_j$ \cite[Definition 9]{Budisic2012AppliedKoopmanism}.
\begin{lemma}[Universality of \eqref{eq:KoopObs}]\label{lem:universal}
    Consider a quantity of interest $h \in C(\Set{X})$, a forward-complete system flow $\bm{F}_{t}(\cdot)$ on a non-recurrent domain $\Set{X}$ (Definition \ref{def:Nonrec}) of a compact set $\Set{X}$. Then, the output trajectory ${y}(t) = {h}(\bm{x}(t)), \forall t \in [0,T]$ is arbitrarily closely described by the eigenpairs $\{\exp{\eig_j t},g_j\}_{j \in \Set{N}} {\subseteq} (\Set{C} \tsgn{\times} C(\Set{X}))$ of the Koopman operator semigroup $\{\operator{A}_t\}^{T}_{t\tsgn{=}0}$ so that $\forall \varepsilon > 0, \exists \bar{D} \in \Set{N}$
    \begin{equation}\label{modeDecom}
 |{h}(\bm{x}(t)) - \textstyle{\sum^{\bar{D}}_{j =1}}\operatorname{e}^{\eig_j t} g_j (\bm{x}\naught) | < \varepsilon, \forall t \in [0,T].
    \end{equation}
\end{lemma}
\begin{proof}
With continuous eigenfunctions for continuous systems proved valid in \cite[Lemma 5.1]{Mezic2020SpectrumGeometry},\cite[Theorem 1]{Korda2020OptimalControl}, the space of continuous functions over a compact set is naturally the space of interest. On a non-recurrent domain, there exist uniquely defined non-trivial eigenfunctions and, by \cite[Theorem 3.0.2]{Kuster2015TheSystems}, the spectrum is rich -- with any eigenvalue in the closed complex unit disk legitimate \citep{Ikeda2022KoopmanSpaces}. Further, by \cite[Theorem 2]{Korda2020OptimalControl}, this richness is inherited by the Koopman eigenfunctions --- making them universal approximators of continuous functions.
\end{proof}
\paragraph{Intuition on spectral sampling}
One may wonder if sampling spectra from a set enclosing the true spectrum may be enough to represent the spectral decomposition of the Koopman operator. Recalling that the spectral decomposition consists of projections to eigenspaces, we remark on a well-known result.
\begin{remark}\label{rmk:ChoiceOfMeasures}
    The choice of our measure of integration might seem arbitrary, and it indeed is. Since we, in general, do not assume knowledge of the spectrum of the Koopman-semigroup, we \textit{have to} make an approximation. To this end, an educated guess on where the (point-) spectrum might be located is helpful. As elaborated above, the Hille-Yosida-Theorem provides a convenient way to connect the practically attainable growth rates to bounds on the spectrum. 
    The Riesz projection operator $P_\eig: \raum{C}\mapsto \{g\in\raum{C}: \operator{A}g=\eig g\}$ to an eigenspace of $\operator{A}$ can be represented by 
    \[P_{\eig} = \frac{1}{2\pi i}\normalint_{\gamma_{\eig}} \frac{\d{s}}{s - \operator{A}},\]
    where $\gamma_{\eig}$ is a Jordan curve enclosing $\eig$ and no other point in $\sigma(\operator{A})$ \citep{Dunford1943SpectralProjections}. %
    Obviously $\bigcup_{\eig\in\sigma(\operator{A})} \operatorname{range}(P_{\eig})=\raum{C}$, iterating on the fact that we can represent the operator $T$ by its spectral components.
    It becomes apparent that sampling from a set enclosing $\sigma(\eig)$ can be seen as sampling curves, eventually enclosing sufficient spectral components. And as stated, one can choose arbitrary measures on $\Complex$ as long as one ensures they enclose the spectrum. 
\end{remark}

\section{REPRESENTATIONAL POWER OF KOOPMAN SPECTRAL KERNELS}\label{supl:repPWR}

When the Koopman operator is spectral, e.g., on a non-recurrent domain, the canonical representation of a Koopman operator acting on a well-specified observable ${h} \in \RKHS$ remains well-specified.
\begin{lemma}
  Denote by $[\mathsf{y}_t^{\txt{\tiny KE}}]_{\sim}$ the $L_2$ equivalence class of $\mathsf{y}_t^{\txt{\tiny KE}}$ and denote 
  \begin{align}
    \mathcal{A}^{[\tau_s,0]}_t = \sum^{\infty}_{j=1}\exp{\eig_j t} \mathcal{E}_{\eig_j}^{[\tau_s,0]}
\end{align}
  as the canonical spectral representation of a Koopman operator on the time-interval $[\tau_s,0]$. If $h_0 \in \RKHS$, there exists a kernel $k_y^{\txt{\tiny KE}}$, with integral operator $\mathcal{T}_{k_y^{\txt{\tiny KE}}}={ \mathcal{A}^{[\tau_s,0]}_t} \mathcal{T}_{k_x} {\mathcal{A}^{[\tau_s,0]}_t}^* $, s.t. for $\mathsf{h} \sim \mathcal{GP}\left(0, k_{x}\right), \mathsf{y}_t^{\txt{\tiny KE}} \sim \mathcal{GP}\left(0, k_y^{\txt{\tiny KE}}\right),[\mathsf{y}_t^{\txt{\tiny KE}}]_{\sim}$ has the same distribution as $ \mathcal{A}^{[\tau_s,0]}_t[\mathsf{h}]_{\sim}$.
  \begin{proof}
      Lemma 3.1 \citep{Wang2022}.
  \end{proof}
\end{lemma}

\begin{remark}
    Note  that the above holds for the \eqref{eq:SDK} when setting the individual equivariance operators to the identity so that  $\mathcal{A}^{[0,0]}_t = \sum^{\infty}_{j=1}\exp{\eig_j t}I_j$ 
\end{remark}

The above infinite sum may seem concerning. However, under mild conditions \ref{asm:Sspec}, there always exists a finite rank representation that is dense in the space of continuous function equipped with the supremum norm (universal). This is formalized in the following.

\begin{lemma}[Universality]
Let \ref{asm:Sspec} hold and consider a universal base kernel $k_x$ so that $\{k_{g_j}= k_x\}^D_{j=1}$. Then the induced kernels \eqref{eq:SDK} and \eqref{eq:KE-SDK} are universal, given a sufficiently rich spectral components $\{\eig_j\in \Set{C}\}^{D}_j $.
    \begin{proof}
       The universality of \eqref{eq:SDK} follows directly by \citet[Theorem 2]{Korda2020OptimalControl}. With a well-specified symmetrization by \ref{asm:Sspec}, the universality for functions satisfying Koopman-equivariance is inherited by applying \citet[Theorem 1 (ii)]{KKR_neurips2023} component-wise.
    \end{proof}
\end{lemma}

In the above lemma, the "sufficiently rich" can be understood as a set of eigenvalues that enclose the true spectrum. This is straightforwardly achieved by sampling eigenvalues from a distribution with support that encloses the true spectra \citep[Proposition 3]{KKR_neurips2023}.

\begin{remark}[Uncountable eigenpairs]\label{rem:InfKEIGS}
    Under Assumption \ref{asm:Sspec} any and all eigenvalues are legitimate, and, for each $\eig_j$, there are at least uncountably infinitely many eigenfunction-eigenvalue pairs \citep[Corollary 3]{Bollt2021GeometricRepresentation}.
\end{remark}

While na\"{i}vely, one would be tempted to optimize a large set of individual eigenvalues, this may be a highly ill-posed problem; as indicated by  Remark \ref{rem:InfKEIGS} and limits the optimization to a very few eigenvalues in practice \citep{Korda2020OptimalControl,caldarelli2024linear}. This is a key motivation in our likelihood optimization of the eigenvalue distribution with only a few degrees of freedom, allows us to efficiently \textit{choose the most likely set of eigenvalues amongst the infinite possibilities}.

\section{PROOFS OF THEORETICAL RESULTS}\label{supl:Proof}

In the following, we restate the definition of Koopman equivariance for completeness.

\begin{definition}[Definition \ref{def:KEIGS} restated]
    Let $[\tau_s,\tau_e] \subset \Set{R}$ be a compact subset of the time axis and $\mathcal{M}$ a manifold. 
    A map $\phi_{\eig}: \mathcal{M} \mapsto \Set{C}^\mathcal{M}$ is called $ [\tau_s,\tau_e]_{\eig}$-{\em Koopman-equivariant} if 
    \begin{align}
        \phi_{\eig}\circ\bm{F}_t=\exp{\lambda t} \phi_\eig
    \end{align}
    on $\mathcal{M}$ for any $t \in [\tau_s,\tau_e]$.
\end{definition}

\subsection{Symmetrization Based on Koopman Equivariance}

\begin{theorem}[Theorem \ref{thm:Symm} restated \& expanded]
Consider the symmetrization operator $\KEop{[\tau_s,\tau_e]}{\eig}: L_{\mu}^2(\mathcal{X}) \rightarrow L_{\mu}^2(\mathcal{X})$ defined as
\begin{align}\label{eq:EqvarEOSupp}
    \KEop{[\tau_s,\tau_e]}{\eig} g := \mathbb{E}_{t \sim \mu{([\tau_s,\tau_e])}}\left[ \operatorname{e}^{-\lambda t} g\left(\bm{x}(t))\right)\right]
\end{align}
so that it is well-defined and self-adjoint. Then,
\begin{enumerate}[leftmargin=*,label=\roman*.]
\item \label{itm:iffKE} a function $f \in L_{\mu}^2(\mathcal{X})$ is $\KEqvar{[\tau_s,\tau_e]}\text{-equivariant}$ if and only if $\KEop{[\tau_s,\tau_e]}{\eig}[f] = f$, implying $\KEop{[\tau_s,\tau_e]}{\eig}$ is a projection operator so $\|\KEop{[\tau_s,\tau_e]}{\eig}\|=1$ if $L_{\mu}^2(\mathcal{X})$ contains any $\KEqvar{[\tau_s,\tau_e]}\text{-equivariant}$ functions ($\|\KEop{[\tau_s,\tau_e]}{\eig}\|=0$ otherwise); 
 \item \label{itm:Odecomp} $L_{\mu}^2(\mathcal{X})$ decomposes into symmetric and anti-symmetric part
$L_{\mu}^2(\mathcal{X})=\mathcal{S}_\eig \oplus \mathcal{S}_\eig^\perp$
where $\mathcal{S}_\eig=\left\{g \in L_{\mu}^2(\mathcal{X}): g\right.$ is  $\KEqvar{[\tau_s,\tau_e]}\text{-equivariant}\}$ and $\mathcal{S}_\eig^\perp=\{g \in L_{\mu}^2(\mathcal{X}):\KEop{[\tau_s,\tau_e]}{\eig}  g=0\}$;
\item \label{itm:Proj}
the symmetrization operator
$\KEop{[\tau_s,\tau_e]}{\eig}$ maps $g$ to the unique solution of
\begin{align}
    \phi_{\eig} = \argmin_{\psi \in \mathcal{S}_\eig} \| g - \psi \|^2_{\mu}.
\end{align}
\end{enumerate}
\begin{proof}
\ref{itm:iffKE} Lemma C.7 Corollary C.8 \citep{elesedy21a} 
\ref{itm:Odecomp},\ref{itm:Proj} Proposition 24 and 25 \citep{elesedy21a} and Proposition 3.9. \citep{Elesedy2023}
\end{proof}
\end{theorem}

\subsection{New Information Gain Rates}
\paragraph{Technical Lemmas}
As our technical results rely on a spectral representation of the base hypothesis space, we state the following technical lemma on Mercer representations considered in this work.

\begin{lemma}[Mercer representation]\label{lem:MercerApp}
Let $\RKHS$ be any RKHS with kernel $k_x$ s.t.~$
\normalint P(d\bm{x}) k_x(\bm{x},\bm{x}) < \infty.
$
Then 
\begin{enumerate}[leftmargin=*,label=\roman*.]
   \item $\RKHS$ can be embedded into $L_2(P(d\bm{x}))$, and 
   the natural inclusion operator $\iota_x:\RKHS\to L_2(P(d\bm{x}))$ and $\iota_x^\top$ are Hilbert-Schmidt; the map $\operator{T}_{k_x}: h\mapsto \normalint P(d\bm{x}) k_x(\bm{x},\cdot) h(\bm{x})$ defines a positive, self-adjoint and trace-class operator; $\operator{T}_{k_x} = \iota_x\iota_x^\top$. 
   \item $\operator{T}_{k_x}$ has the decomposition $$
\operator{T}_{k_x} h = \sum_{i\in I} \mu_i \langle\bar{\varphi}_i, h \rangle_2\bar{\varphi}_i,
$$
where the index set $I\subset \Set{N}$ is at most countable, and $\{\bar e_i\}$ is an orthonormal system in $L_2(P(d\bm{x}))$. 
\item There exists an orthogonal system $\{e_i: i\in I\}$ of $\RKHS$ s.t.~$[e_i]_\sim = \sqrt{\lambda_i}\bar{\varphi}_i$. 
\item If $k_x$ is additionally bounded and continuous, 
$\{e_i: i\in I\}$ will %
define a Mercer's representation whose convergence is absolute and uniform. %
\end{enumerate}
\begin{proof}
\cite[Lemma 2.3, 2.2 (for i), %
       2.12 (for ii-iii), Corollary 3.5 %
       (for iv)]{Steinwart2012MercersTO}.
\end{proof}
\end{lemma}

We state the following technical Lemma that will help prove a result on information gain rates.

\begin{lemma}[\citep{Bhatia1997}]\label{lem:CompDecay}
    Let $\operator{A},\operator{B}$ be any two operators, $\|\cdot\|$ denote the operator norm and $s_j$ the $j$-th largest singular value. Then
    $$
    s_j(\operator{A}\operator{B}) \leq \min\{\|\operator{B}\|s_j(\operator{A}),\|\operator{A}\|s_j(\operator{B})\}
    $$
\end{lemma}

The above results will be used to bound the eigenvalue decay i.e. 
$$
\eig_j(\operator{A}\operator{B}\operator{B}^*\operator{A}^*)=s_j(\operator{A}\operator{B})^2.
$$ 

\begin{theorem}[Theorem \ref{th:infogain_asymp} restated]
Consider the Mercer eigenvalues $\{\mu_j\}^{\infty}_{j=1}$ for $k_x$ and let Assumptions \ref{asm:Hreg},\ref{asm:Sspec} and \ref{asm:Areg} hold. Then $\exists \theta \geq 1$ for\vspace{-1em}
\begin{description}[style=multiline, leftmargin=3em,font=\normalfont]
    \item[\namedlabel{case:PD2}{(\txt{Poly})}] $$\eig_j(A_1)\lesssim j^{-p} \wedge \mu_j\lesssim j^{-a}, a>1$$ or
    \item[\namedlabel{case:ED2}{(\txt{Exp})}]  $$\eig_j(A_1)\lesssim \exp{-j^p} \wedge~\mu_j \lesssim  \exp{-j^b},b>0$$ so that
\end{description}
\[
    \gamma^{\sigma}_N\left(k^{\txt{KE}}_{y}\right)  \in
    \tilde{\mathcal{O}}(\left(\gamma^{\sigma}_N(k_{x})\right)^{\nicefrac{1}{\theta}})
\]
where $\theta=\frac{\max\{2p,a\}}{a}$ \ref{case:PD2} and $\theta=\frac{\max\{p,b\}}{b}$ \ref{case:ED2}.
\begin{proof}
We prove the results considering the following two eigendecay profiles:
    \begin{description}[style=multiline, leftmargin=3em,font=\normalfont]
    \item[\ref{case:PD2}{(\txt{Poly})}] 
First, based on the information gain results from \cite[Corollary 1]{pmlr-v130-vakili21a}, we can use a simplified (free of constants) information gain 
\begin{align}\label{eq:polyIGsimple}
\gamma^{\sigma}_N\left(k_{x}\right) \in \mathcal{O}\left( N^{\frac{1}{a}} {(\log{N})}^{1-\frac{1}{a}} \right)
\end{align}
for the base kernel $k_x$ with a polynomial decay of Mercer eigenvalues
\begin{align}
    \mu_j\lesssim j^{-a}, a>1.
\end{align}
Secondly,  by boundedness of $A_1$, we have an $A_1$-induced decay $\eig_j(A_1 \mathcal{T}_{k_x}A^*_1)\lesssim j^{-a^\prime}$ for some $a'$. This allows us to define a ratio between them and the native decay rate of Mercer eigenvalues $\theta:=\frac{a^\prime}{a} > 0$.  To uncover the information gain differences, we
\begin{subequations}
can equivalently define $a^\prime = \theta a$, such that expressing everything in terms of $\theta$ and $a$ we get  
    \begin{align}
\gamma^{\sigma}_N\left(k^{\txt{KE}}_{y}\right) 
& \in {\mathcal{O}} \left({N}^{\frac{1}{a\theta}}(\log N)^{(1
-\frac{1}{a\theta})}\right) \\
& \in {\mathcal{O}} \left({N}^{\frac{1}{a}(\frac{1}{\theta})}(\log N)^{\left(\frac{1}{\theta}+(1-\frac{1}{\theta})
-\frac{1}{a \theta}\right)}\right) \\
& \in {\mathcal{O}} \left({N}^{\frac{1}{a}(\frac{1}{\theta})}(\log N)^{(1
-\frac{1}{a})(\frac{1}{\theta})}(\log N)^{1-\frac{1}{\theta}}\right) \\
& \in {\mathcal{O}} \left(\left({N}^{\frac{1}{a}}(\log N)^{(1
-\frac{1}{a})}\right)^\frac{1}{\theta} (\log N)^{1-\frac{1}{\theta}}\right)\\
& \in \gamma^{\sigma}_N\left(k_{x}\right)^\frac{1}{\theta}{\mathcal{O}} \left( (\log N)^{1-\frac{1}{\theta}}\right) \\
& \in \tilde{\mathcal{O}} \left(\left(\gamma^{\sigma}_N\left(k_{x}\right)\right)^\frac{1}{\theta}\right).
\end{align}
\end{subequations}
Using $\eig_j(A_1)\lesssim j^{-p}$ and $\eig_j(\iota_x\iota_x^*)\overset{\text{Lem. \ref{lem:MercerApp}}}{\equiv}\eig_j\left(\operator{T}_{k_x}\right):=\mu_j\lesssim j^{-a}, a>1$  and invoking Lemma \ref{lem:CompDecay}, we obtain 
\begin{align}
\eig_j\left({{A}^{}_1} \iota_x\iota_x^* {{A}^{*}_1}\right) = \left(s_j\left({{A}^{}_1} \iota_x\right)\right)^2 &  \leq  \left(\min\{\|\iota_x\|s_j(A_1),\|A_1\|s_j(\iota_x)\}\right)^2\\
&  \lesssim \left(\min\{s_j(A_1),s_j(\iota_x)\}\right)^2 \\
&  \lesssim \left(\min\left\{j^{-{p}},\sqrt{j^{-{a}}}\right\}\right)^2 \\
&  \lesssim \left(j^{-{\max\{p,{\frac{a}{2}}\}}}\right)^2\\
     & \lesssim  {j^{-\max\{2p,a\}}}
\end{align}
leading to
\begin{align}
    \gamma^{\sigma}_N\left(k^{\txt{KE}}_{y}\right) & \in \tilde{\mathcal{O}} \left(\gamma^{\sigma}_N\left(k_{x}\right)^\frac{a}{\max\{2p,a\}}\right).
\end{align}
where identifying $\theta := \frac{\max\{2p,a\}}{a} \geq 1$ proves the \ref{case:PD2} part of the result.
    \item[\ref{case:ED2}{(\txt{Exp})}]  
First, based on the information gain results from \cite[Corollary 1]{pmlr-v130-vakili21a}, we can use a simplified (free of constants) information gain 
\begin{align}\label{eq:expIGsimple}
\gamma^{\sigma}_N\left(k_{x}\right) \in \mathcal{O}\left({(\log{N})}^{1+\frac{1}{b}} \right)
\end{align}
for the base kernel $k_x$ with an exponential decay of Mercer eigenvalues
\begin{align}
    \mu_j \lesssim \exp{{-{j}^{b}}},b>0.
\end{align}
Secondly, by boundedness of $A_1$, we have an $A_1$-induced decay $\eig_j(A_1 \mathcal{T}_{k_x}A^*_1)\lesssim \exp{{-{j}^{b^\prime}}}$. This allows us to define a ratio between them and the native decay rate of Mercer eigenvalues $\theta:=\frac{b^\prime}{b} > 0$.  To uncover the information gain differences,
\begin{subequations}
we can equivalently define $b^\prime = \theta b$, such that expressing everything in terms of $\theta$ and $b$ we get 
    \begin{align}
\gamma^{\sigma}_N\left(k^{\txt{KE}}_{y}\right) 
& \in {\mathcal{O}} \left((\log N)^{(1
+\frac{1}{b \theta})}\right) \\ 
& \in {\mathcal{O}} \left((\log N)^{(\frac{1}{\theta}+(1-\frac{1}{\theta})
+\frac{1}{b \theta})}\right)  \\
& \in {\mathcal{O}} \left((\log N)^{(1
+\frac{1}{b})\frac{1}{\theta}} (\log N)^{1-\frac{1}{\theta}}\right)\\
& \in {\mathcal{O}} \left(\left((\log N)^{(1
+\frac{1}{b})}\right)^\frac{1}{\theta} (\log N)^{1-\frac{1}{\theta}}\right)\\
& \in {\mathcal{O}} \left((\log N)^{(1
+\frac{1}{b})}\right)^\frac{1}{\theta} {\mathcal{O}} (\log N)^{1-\frac{1}{\theta}}\\
& \in \gamma^{\sigma}_N\left(k_{x}\right)^\frac{1}{\theta} {\mathcal{O}} (\log N)^{1-\frac{1}{\theta}}\label{eq:ExpFullIG}\\
& \in \tilde{\mathcal{O}} \left(\gamma^{\sigma}_N\left(k_{x}\right)^\frac{1}{\theta}\right).
\end{align}
\end{subequations}
Using $\eig_j(A_1)\lesssim \exp{{-{j}^{p}}}$ and $\eig_j(\iota_x\iota_x^*)\overset{\text{Lem. \ref{lem:MercerApp}}}{\equiv}\eig_j\left(\operator{T}_{k_x}\right):=\mu_j\lesssim \exp{{-{j}^{b}}}, b>0$  and invoking Lemma \ref{lem:CompDecay}, we obtain 
\begin{align}
\eig_j\left({{A}^{}_1} \iota_x\iota_x^* {{A}^{*}_1}\right) = \left(s_j\left({{A}^{}_1} \iota_x\right)\right)^2 &  \leq  \left(\min\{\|\iota_x\|s_j(A_1),\|A_1\|s_j(\iota_x)\}\right)^2\\
&  \lesssim \left(\min\{s_j(A_1),s_j(\iota_x)\}\right)^2 \\
&  \lesssim \left(\exp{-{\max\{j^p,\frac{1}{2}j^{{b}}\}}}\right)^2 \\
&  \lesssim \exp{-{\max\{2j^p,j^{{b}}\}}} \\
     & \lesssim  \exp{-j^{\max\{p,b\}}}
\end{align}
leading to
\begin{align}
    \gamma^{\sigma}_N\left(k^{\txt{KE}}_{y}\right) &  \in \tilde{\mathcal{O}} \left(\gamma^{\sigma}_N\left(k_{x}\right)^\frac{b}{\max\{p,b\}}\right).
\end{align}
where identifying $\theta := \frac{\max\{p,b\}}{b} \geq 1$ proves the \ref{case:ED2} part of the result, finishing the proof.
\end{description}
\end{proof}
\end{theorem}

\begin{remark}[Base SE kernel in \eqref{eq:KE-SDK}]
    By plugging in the known input-dimension dependent information gain for the SE kernel in \eqref{eq:ExpFullIG}, we see that the decay speed-up  $\theta > 1 \equiv 2p > b$ can counteract the curse of input dimensionality, leading to to 
    $$
    \gamma^{\sigma}_N\left(k^{\txt{KE}}_{y}\right) \in {\mathcal{O}} \left((\log N)^{\frac{n}{\theta}+1}\right).
    $$
\end{remark}

\paragraph{Intuition on time-independent $A_1$} 
We defined the operator $A_1$ in \ref{asm:Areg} as time-independent for ease of exposition; the family of $\{A_t\}^T_{t=0}$ is uniquely defined by an infinitesimal generator.
This is due to the confinement to a non-recurrent domain \ref{asm:Sspec} that prescribes finite frequencies of oscillation and finite growth rates, allowing for the infinitesimal generator bounded \citep{Zeng2023AOperator}.
In practice, we can always take the worst-case time-exponent for analysis by scaling the time appropriately.

\section{DETAILS ON NUMERICAL EXPERIMENTS}\label{supl:NumExp}

\subsection{Implementation Details}
We implemented GP regression using \texttt{GPJax}~\citep{Pinder2022}. All of the experiments were performed on machines with {2TB} of RAM, 8 NVIDIA Tesla P100 16GB GPUs and 4 AMD EPYC 7542 CPUs.

\subsubsection{Spectral Hyperprior}\label{ssec:SpecPrio}
To tractably optimize over a spectral prior, we use the noise transfer (outsourcing) trick by \cite[Theorem 5.10]{kallenberg1997foundations} to model the eigenvalue distribution $p(\lambda) \approx \rho_{}(\bm{\vartheta})$.
This choice limits the number of required parameters since $\bm{\vartheta}$ has fewer parameters (degrees of freedom) than the number of eigenspaces $\|\bm{\vartheta}\|_0 \ll |D|$. Furthermore, it allows for the use of log-likelihood maximization just like with any other set of hyperparameters. We use a uniform distribution on $\{\lambda_j=s_j+\mathrm{i}\omega_j: s\in [-\vartheta_{s}, \vartheta_{s}]+\vartheta_{\overline{s}},~\omega\in[-\vartheta_{\omega}, \vartheta_{\omega}]+\vartheta_{\overline{\omega}}\}$,~$\bm{\vartheta}{=}[\vartheta_s,\vartheta_{\overline{s}},\vartheta_{\omega},\vartheta_{\overline{\omega}}]^\top$. To obtain equivariant features, we compute the expectation~\eqref{eq:EqvarEOSupp} wrt. a uniform underlying distribution in time, which in case of equally spaced points in time leads to a trapezoid rule.
\subsubsection{Preprocessing and Initialization}
We standardize all data trajectories such that the target has to have zero mean and unit variance and the forecast time is between zero and one. This allows us to choose similar parameters for all datasets. We initialize the generative parameters as follows for KE-GP and C-GP:
\begin{enumerate}[leftmargin=*,label=\txt{init}.\roman*)]
    \item \textit{Prior mean}: $\mu(\bm{x})=0$ and keep it fixed;
    \item \textit{Lengthscale}: $\frac{\sqrt{n_x}}{2}\operatorname{std} (\Set{X}_{\text{input}})$ following \cite{pmlr-vanillaBOgreat-hvarfner24a};
    \item \textit{Signal variance}: $\sigma_{s}^2=1$ and fix it, since the data was standardized. Is advised based on the results of \cite{pmlr-vanillaBOgreat-hvarfner24a};
    \item \textit{Observation noise variance}: $\sigma_{\txt{on}}^2=1$;
    \item\label{itm:SpecPrio} \textit{Spectral prior}:  $\vartheta_{{\omega}}=15, \vartheta_{\overline{\omega}}=0$, $\vartheta_{s}=1$, $\vartheta_{\overline{s}}=0$ for scale and bias parameters of the uniform distribution.
    \item \textit{Inducing trajectories}: $N{=}32$ for exact GP and variational inference, sampled form the train split.
\end{enumerate}

\subsubsection{Variational Inference}
As variational inference is notoriously sensitive to initial guesses, we employ a scheme to robustly optimize the generative and variational parameters.
To this end we first get reasonable guesses as described above, sample a set of inducing points $(\tilde{Z}, \tilde{Y})$ from the training data and train an exact GP on the inducing inputs via marginal log-likelihood. This yields a set of generative parameters, in particular parameters for the spectral distribution that fit the data. The so obtained posterior is used to initialize the variational GP: $m = \overline{\mu}_{\text{MLL}}(\tilde{Z})$ and $S = \overline{\sigma}_{\text{MLL}}(\tilde{Z}, \tilde{Z})$. This means the initial VI model is the exact posterior on inducing inputs, making the optimization easier, as the initial gradients are smaller. To optimize the variational GP we follow~\cite{Toth20a} and start by first optimizing the variational parameters only, then optimize all parameters jointly and finally optimize the variational parameters on the joint training and validation dataset.

Due to the structure of our Koopman-equivariant construction, and the resulting benefits in information gain presented in Section \ref{section:analysisofsamplecomplexity}, we found KE-GPs more robust to a lack of correlation between points than C-GP, an issue commonly observed in conventional sparse GP approximations \citep{Murray2010,HensmanNIPS2015}. In particular, a lower information gain implies that less inducing points are required than with conventional GPs to accurately represent the full posterior \citep{Burt2019RatesRegression}.%

\paragraph{C-GP}
For the contextual GP \citep{Li2024STkernel} with covariance $k^{\txt{SE}}(t,t^\prime)\otimes k^{\txt{SE}}(\bm{x}_0,\bm{x}^\prime_i)$, we perform the same type of varionational inference as described above, but consider inducing points $\bm{z}=[\bm{x}_0^\top,t^\top]^\top$.

\subsection{Benchmark Dynamics}
We perform our quantitative study on the following examples of varying complexity. 
From the robotics domain, we consider expert demonstrations from D4RL~\citep{fu2020d4rl} from the \txt{\small halfcheetah} environment and forecast the first state and action. We take temperature data from the {Monash TSF} benchmark~\citep{godahewa2021monash} as a sample for highly complex weather dynamics. Since the latter datasets provide a single long trajectory, we split off the last chunk as test data and partition the trajectory into \#$N$ dataparis pairs to comply with \eqref{eq:dataTraj}.

\paragraph{Predator-Prey ODE} We use the predator-prey model:
\begin{equation}\label{eq:volterra_lotka}
    \dot{x}_{1} = r_{1}  x_{1} + c_{i}  \gamma_{1} x_{1}  x_{1}
,\quad
\dot{x}_{2} = r_{2}  x_{2} + c_{i} \gamma_{2} x_{1}  x_{2}
.
\end{equation}%
where $r_1,r_2$, $\gamma$, $c_1,c_2$ are reproduction rates, interaction effects and frequency, respectively. We choose parameters $r_1{=}0.2$, $\gamma_1{=}0.4$, $r_2{=}0.25$, $\gamma_2{=}0.2$, $c_i=2$. We create a dataset by simulating the system for $H=64$ steps with $\Delta t=3$ for $N=1024$ trajectories from initial conditions in $[0, 2]\times [0, 1]$.

\paragraph{Linear ODE}
To validate the information gain results on a simple example we use a 2-dimensional linear system $\dot{x_1}=-6 x_2, \dot{x_2}=6x_1$. We create a dataset by simulating the system for $H=16$ steps with $\Delta t=0.06$ for $N=1000$ trajectories from initial conditions in the unit box. The eigenvalues of this system are $\lambda_{1,2}=\pm 6j$. Which we use to compare a randomly sampled enclosing vs. a matching spectral distribution.

\paragraph{D4RL}
 The Datasets for Deep Data-Driven Reinforcement Learning (D4RL) \citep{fu2020d4rl} provides a trajectory collection of reinforcement learning agents interacting with the environments defined in the OpenAI gym  \citep{oAIgym}. We pick the environment \txt{\small halfcheetah} with the expert policy. For the \txt{\small halfcheetah} we use all actions and observations as our state, demonstrating that our method works in a high dimensional input space. The task for the \txt{\small halfcheetah} is to forecast the first action. Input trajectory have 16 steps, the target is to be forecast for 16 steps.
 
\paragraph{Oikolab Temparature}
As a final benchmark, we draw on the \emph{monash\_tsf} dataset collection \citep{godahewa2021monash} and choose their OIKOLAB weather station dataset to test our KE-GP method on the task temperature forecasting. To this end, we provide the models with state trajectories consisting of 8 quantities, including temperature for the past 32 hours and set the task as forecasting the temperature for 16 hours.

\section{ADDITIONAL EXPERIMENTS}\label{supl:AddExp}

\subsection{Covariance Visualisation vs C-GP}\label{sec:covvis}
We compare initial and marginal log-likelihood optimized covariances for our KE-GP and C-GP. The spatial covariance corresponds to that of a trajectory. The temporal covariance is separately displayed for the same trajectory. As Figure \ref{fig:illustrative} shows, the learned covariances for our KE-GP and C-GP are of similar shape, the initial spatial covariance of KEGP is less local than that of MTGP, already encoding the trajectory structure before hyperparameter optimization. Further, the KEGP temporal covariance delivers a considerably simple forecasting model, as it is the superposition of multiple one-dimensional LTI systems \eqref{eq:KoopObs}, as opposed to a spatially and temporally nonlinear covariance of C-GP. 

Notably, since the spectral distribution is a parametrized uniform distribution, KE-GP has only seven parameters, whereas C-GP has 100.

{
\newcommand{\plottimecov}[1]{
\includegraphics[width=\linewidth, trim={0, 1cm, 0, 3cm},clip=true]{#1}
}
\newcommand{\plotspacecov}[1]{
\includegraphics[width=\linewidth, trim={0, 1cm, 0, 4cm},clip=true]{#1}
}
\begin{figure*}[t]
    \vspace{-0.2cm}
    \centering
    { \setlength{\tabcolsep}{3pt}
    \newcolumntype{Y}{>{\centering\arraybackslash}X}
    \begin{tabularx}{0.95\textwidth}{Y|YY|YY}
   \toprule     
         &\multicolumn{2}{c|}{{Space}}
        &\multicolumn{2}{c}{{Time}}
        \\
        \midrule
        \multicolumn{1}{c|}{\multirow{8}{*}{\tbm{KE-GP} (ours)}}  &Initial   & NLL-optimized  & Initial  & NLL-optimized 
        \\
        &\plotspacecov{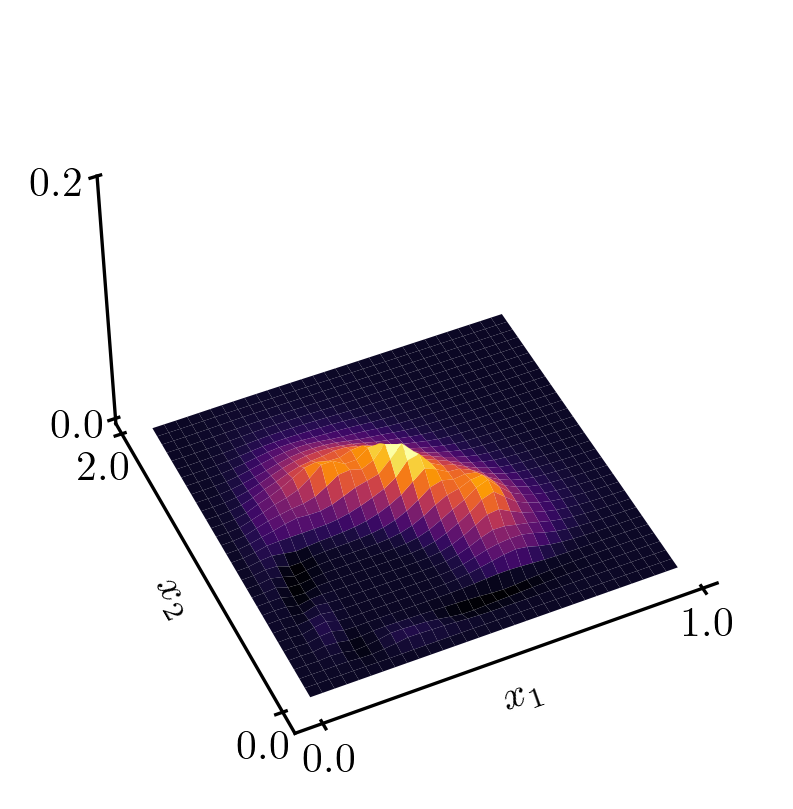}
        &\plotspacecov{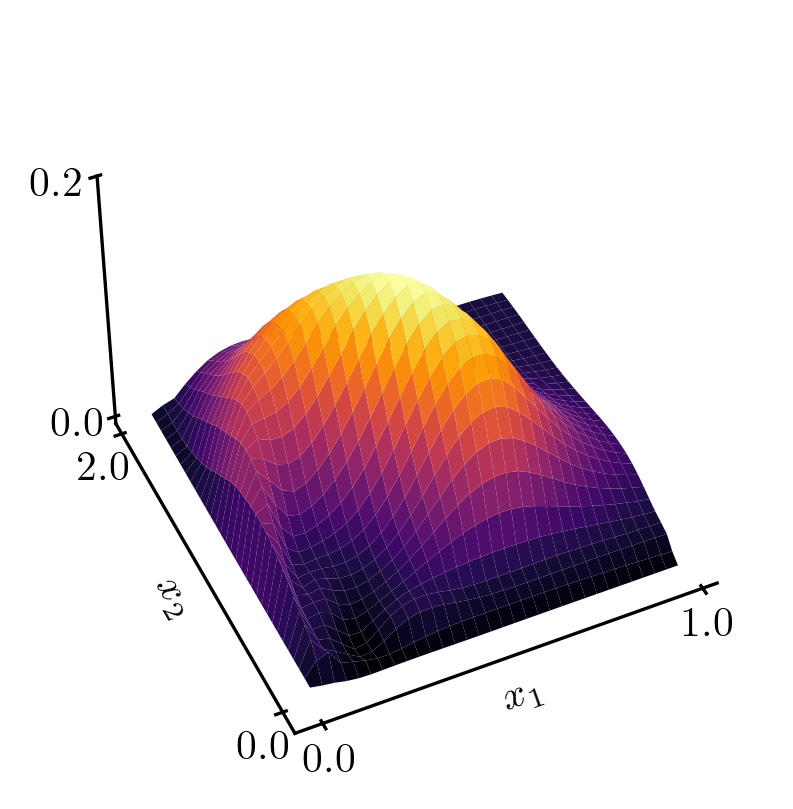}
        &\plottimecov{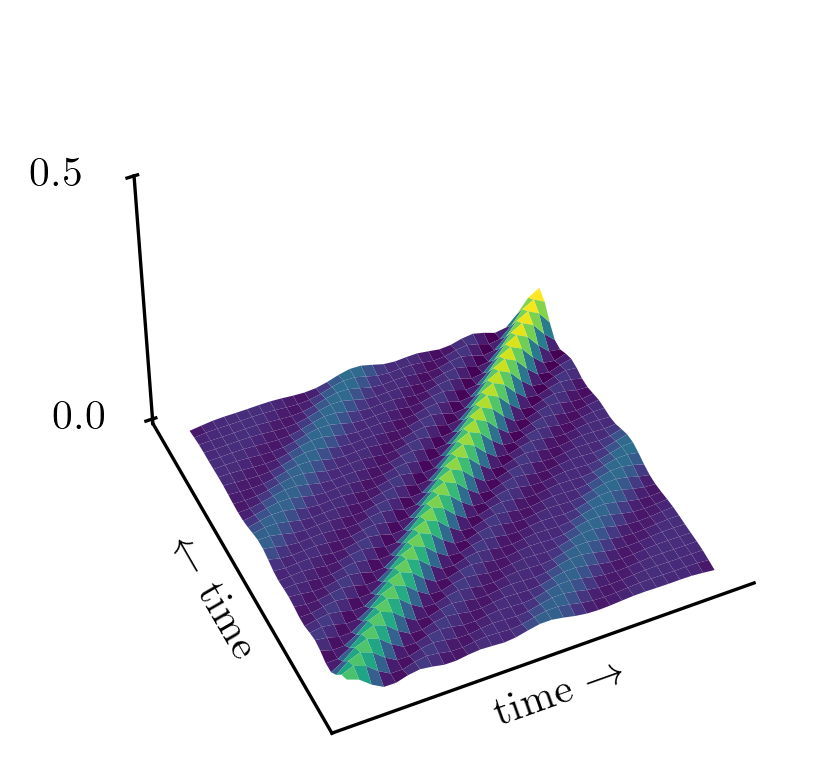}
        &\plottimecov{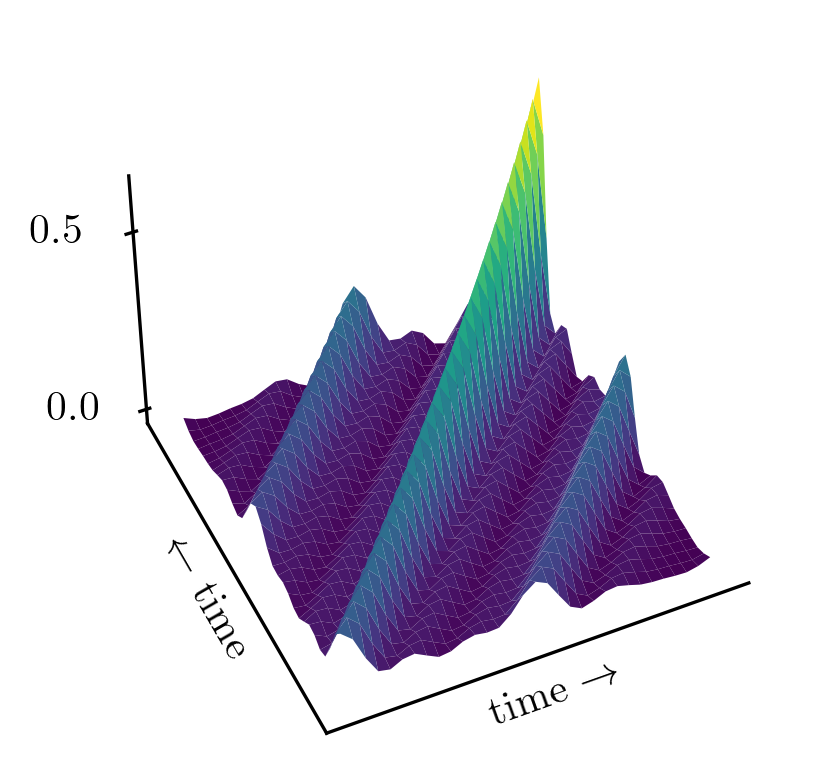}
        \\\midrule
        \multicolumn{1}{c|}{\multirow{8}{*}{C-GP}}&Initial  & NLL-optimized  & Initial  & NLL-optimized 
        \\
        &\plotspacecov{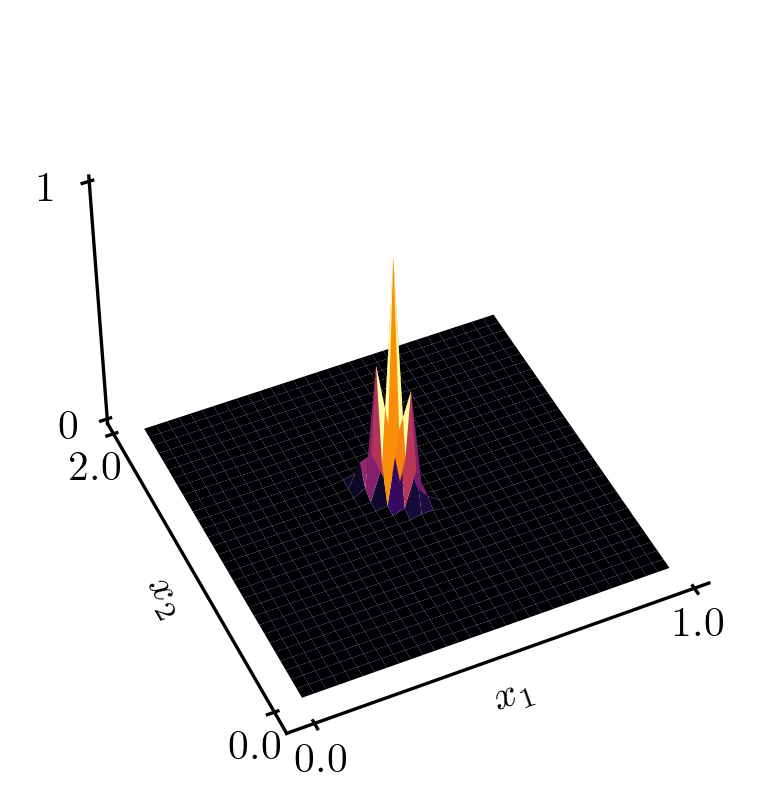}
        &\plotspacecov{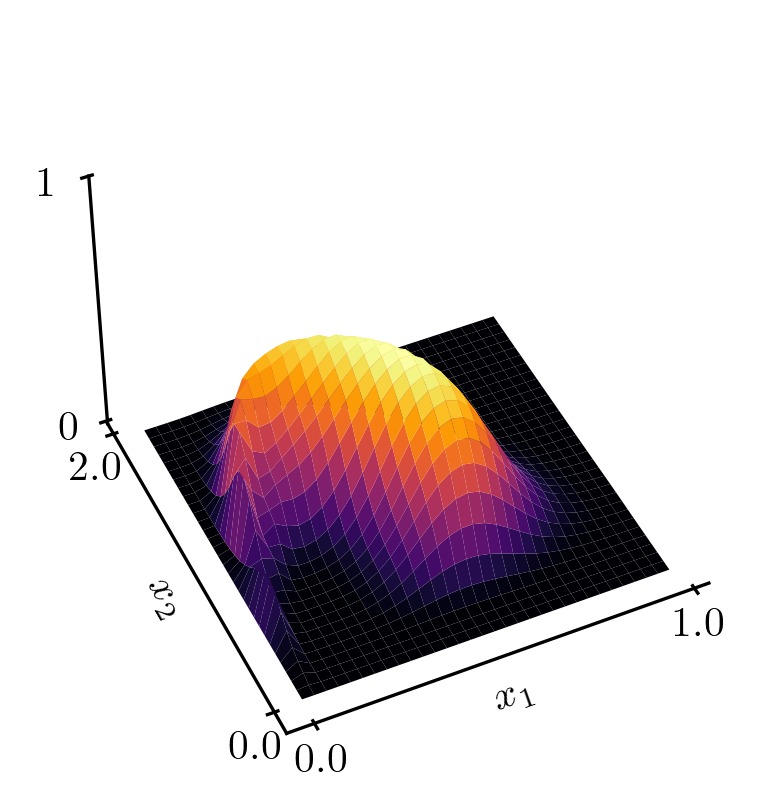}     
        &\plottimecov{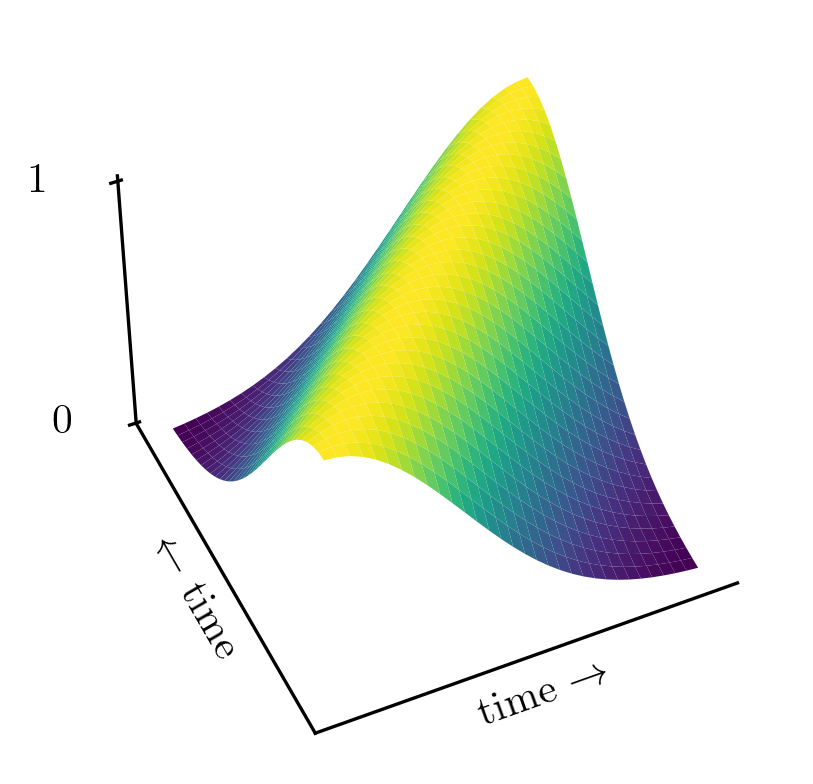}
        &\plottimecov{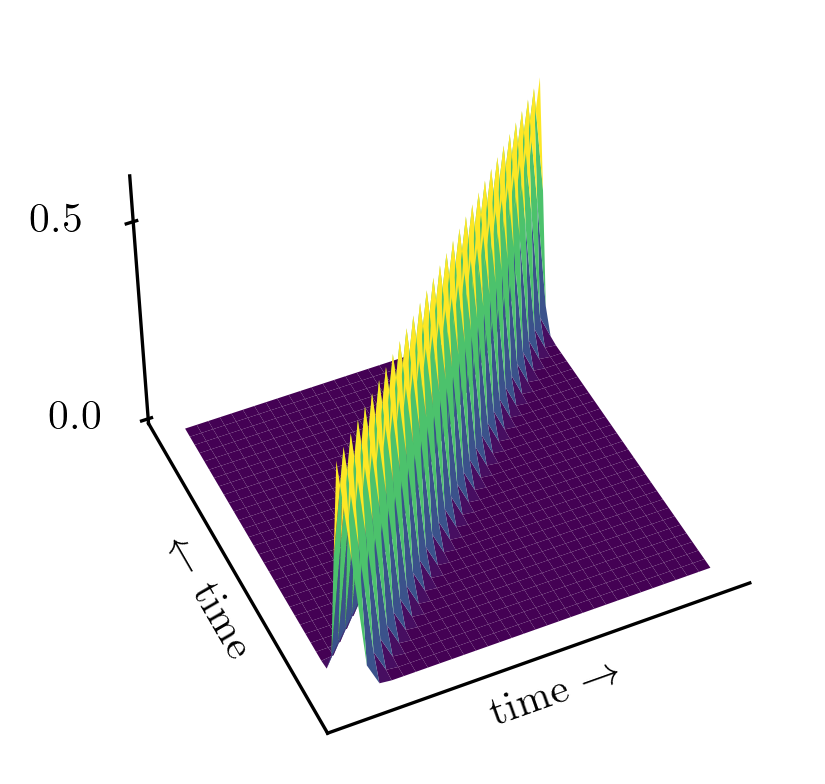}  
         \\
        \bottomrule
    \end{tabularx}}
\label{tab:visualization_covariances}
\caption{Visualization of the GP covariances in space and time. The spatial, KE-GP prior already strongly indicates the shape of the NLL-optimized covariance.}
\end{figure*}%
}%

\subsection{Ablation Studies}
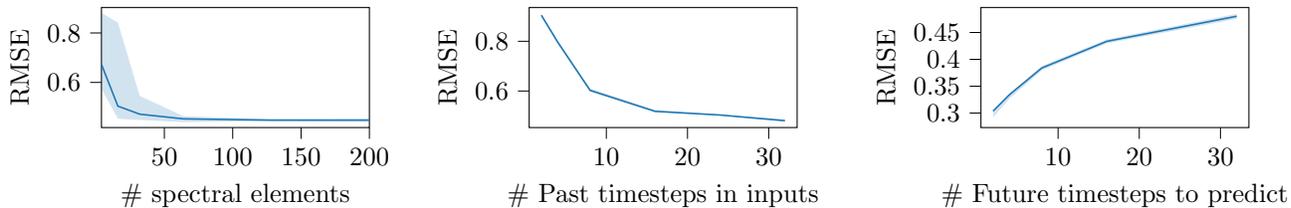
\begin{figure}
    \centering
    \begin{tabular}{ccc}
         \begin{tikzpicture}

\definecolor{darkgray176}{RGB}{176,176,176}
\definecolor{steelblue31119180}{RGB}{31,119,180}

\begin{axis}[
tick align=outside,
tick pos=left,
x grid style={darkgray176},
xlabel={\# spectral elements},
xmin=4, xmax=200,
xtick style={color=black},
y grid style={darkgray176},
ylabel={RMSE},
ymin=0.417490166140613, ymax=0.903106457942425,
ytick style={color=black},
width=.3\textwidth,
height=.3\textwidth/1.618
]
\path [fill=steelblue31119180, fill opacity=0.2]
(axis cs:4,0.881032990133251)
--(axis cs:4,0.574420399529126)
--(axis cs:16,0.453103807689859)
--(axis cs:32,0.448052204935505)
--(axis cs:64,0.439563633949786)
--(axis cs:128,0.443391971510985)
--(axis cs:200,0.439848155491176)
--(axis cs:200,0.452136814706208)
--(axis cs:200,0.452136814706208)
--(axis cs:128,0.45175971054538)
--(axis cs:64,0.463454218665153)
--(axis cs:32,0.545803970532105)
--(axis cs:16,0.842368863109215)
--(axis cs:4,0.881032990133251)
--cycle;

\addplot [semithick, steelblue31119180]
table {%
4 0.673066022018499
16 0.504251898590968
32 0.471655720792867
64 0.452680720917626
128 0.447490430328974
200 0.44762024853612
};
\end{axis}

\end{tikzpicture}&\begin{tikzpicture}

\definecolor{darkgray176}{RGB}{176,176,176}
\definecolor{steelblue31119180}{RGB}{31,119,180}

\begin{axis}[
tick align=outside,
tick pos=left,
x grid style={darkgray176},
xlabel={\# Past timesteps in inputs},
xmin=0.5, xmax=33.5,
xtick style={color=black},
y grid style={darkgray176},
ylabel={RMSE},
ymin=0.452980364289388, ymax=0.935912556203044,
ytick style={color=black},
width=.3\textwidth,
height=.3\textwidth/1.618
]
\path [fill=steelblue31119180, fill opacity=0.2]
(axis cs:2,0.910481528889501)
--(axis cs:2,0.900811607670182)
--(axis cs:4,0.790759293978888)
--(axis cs:8,0.597137146766)
--(axis cs:16,0.516293487479056)
--(axis cs:24,0.502189003713974)
--(axis cs:32,0.47711683083742)
--(axis cs:32,0.486140930745907)
--(axis cs:32,0.486140930745907)
--(axis cs:24,0.504629999074906)
--(axis cs:16,0.524275867431541)
--(axis cs:8,0.610472744161038)
--(axis cs:4,0.801846192521429)
--(axis cs:2,0.910481528889501)
--cycle;

\addplot [semithick, steelblue31119180]
table {%
2 0.904730247622836
4 0.797395625190662
8 0.602611666975664
16 0.518120766090444
24 0.503434577634483
32 0.480596655387911
};
\end{axis}

\end{tikzpicture}
         & \begin{tikzpicture}

\definecolor{darkgray176}{RGB}{176,176,176}
\definecolor{steelblue31119180}{RGB}{31,119,180}

\begin{axis}[
tick align=outside,
tick pos=left,
unbounded coords=jump,
x grid style={darkgray176},
xlabel={\# Future timesteps to predict},
xmin=0.5, xmax=33.5,
xtick style={color=black},
y grid style={darkgray176},
ylabel={RMSE},
ymin=0.273050360799735, ymax=0.496544441562461,
ytick style={color=black},
width=.3\textwidth,
height=.3\textwidth/1.618
]
\path [fill=steelblue31119180, fill opacity=0.2]
(axis cs:2,0.305484449202977)
--(axis cs:2,0.291764398414827)
--(axis cs:4,0.326508472805897)
--(axis cs:8,0.379748969921021)
--(axis cs:16,0.429820280688766)
--(axis cs:32,0.475163641188379)
--(axis cs:32,0.484630805937767)
--(axis cs:32,0.484630805937767)
--(axis cs:16,0.437145931625496)
--(axis cs:8,0.3875434729257)
--(axis cs:4,0.336330751295796)
--(axis cs:2,0.305484449202977)
--cycle;

\addplot [semithick, steelblue31119180]
table {%
2 0.303780062839106
4 0.334069058751524
8 0.383930406359925
16 0.433497393477944
32 0.480176740701709
};
\end{axis}

\end{tikzpicture}
    \end{tabular}
    \caption{Ablation Studies: test RMSE with varying data and spectral parameters; we report mean and interquartile range over 10 runs.}
    \label{fig:abationES}
\end{figure}
As the KE-GP model performance depends on parameters that are absent in naive models, we perform ablation studies to shed light on the role of the new components. In particular, we take a closer look at the effect of the number of modes, the effect of varying length input trajectories, and the effect of varying length forecast time. To this end, we take the predator-prey dynamics \eqref{eq:volterra_lotka}, and 
build exact KE-GP models on $N_{\text{train}}=32$ trajectories and evaluate the model on $N_{\text{test}}=256$ validation trajectories. We do not optimize the base kernel $k_{g_j}$ hyperparameters, i.e., to enable a consistent comparison between models. 

\textbf{Eigenspace Dimensionality} As displayed in Figure~\ref{fig:abationES} (left), the performance improves significantly when increasing the number of eigenspaces sampled up to about $D=64$; by increasing $D$ further, the performance increase saturates. This behavior resembles the increase in resolution we gain by increasing the number of eigenspaces sampled. A reasonable $D$ is dependent on the quality of the learned hyperprior (Subsection \ref{ssec:SpecPrio}), which in turn depends on the dynamical system at hand. Crucially, there is no loss of performance when there are too many spectral elements.

\textbf{Input Trajectory Length} Figure~\ref{fig:abationES}~(middle) shows that a longer input trajectory leads to the equivariance operator $\mathcal{E}$ introducing more information about the dynamics into the prior covariance. As prescribed by representation theory of Section \ref{supl:repPWR} and discussed in~\ref{sec:covvis}, this leads to a better prior even without the need for optimization.

\textbf{Target Trajectory Length} As displayed in Figure~\ref{fig:abationES}~(right), the forecasting problem gets progressively harder as we want to forecast for a longer time, the RMSE increases at a rate of approximately $\sqrt{\txt{steps}}$, which is expected for RMSE.%

\newpage
\bibliography{References}

\end{document}